\documentclass[lettersize,journal]{IEEEtran}

\usepackage{url}

\usepackage{tikz}
\usepackage{newfloat}
\usepackage{listings}
\usepackage{bm}
\usepackage{multirow}
\usepackage{booktabs}
\usepackage{svg} 
\usepackage{mathrsfs}
\usepackage{amsfonts}
\usepackage{multicol}
\usepackage{amsthm}
\usepackage{amsmath}
\usepackage{subfigure}
\usepackage{array}
\usepackage{xcolor, colortbl}
\usepackage{makecell}
\usepackage[colorlinks=true, citecolor=blue, linkcolor=blue, urlcolor=blue]{hyperref}
\usepackage{orcidlink}

\newtheorem{lemma}{Lemma}
\newtheorem{proposition}{Proposition}

\usepackage[linesnumbered,ruled,vlined]{algorithm2e}
\usepackage{algorithmic}

\newcommand{\rka}[1]{\cellcolor{gray!60}{#1}}

\newcommand{\rkat}[1]{\colorbox{gray!60}{#1}}
\newcommand{\rkb}[1]{\cellcolor{gray!30}{#1}}
\newcommand{\rkbt}[1]{\colorbox{gray!30}{#1}}

\newcommand{\MobileSignal}[1]{
    \begin{tikzpicture}[scale=0.5]
        \definecolor{gentlegreen}{RGB}{130, 179, 102}
        
        \foreach \i in {1,2,3,4,5} {
            \pgfmathsetmacro\xleft{(\i-1) * 0.3}
            \pgfmathsetmacro\xright{(\i-1) * 0.3 + 0.2}
            \pgfmathsetmacro\yheight{\i * 0.1 + 0.1}
            
            \ifnum\i>#1
                \fill[white,draw=black] (\xleft,0) rectangle (\xright,\yheight); 
            \else
                \fill[gentlegreen,draw=black] (\xleft,0) rectangle (\xright,\yheight); 
            \fi
        }
    \end{tikzpicture}
}

\begin{document}

\title{Coden: Efficient Temporal Graph Neural Networks for Continuous Prediction}

\author{Zulun Zhu\,\orcidlink{0000-0002-5176-6378},
        Siqiang Luo\,\orcidlink{0000-0001-8197-0903} , ~\IEEEmembership{Member,~IEEE}
        \thanks{The authors are with the College of Computing and Data Science,
  Nanyang Technological University, Singapore.
  E-mail: ZULUN001@ntu.edu.sg; siqiang.luo@ntu.edu.sg.
}}

\markboth{Journal of \LaTeX\ Class Files,~Vol.~14, No.~8, August~2021}%
{Shell \MakeLowercase{\textit{et al.}}: A Sample Article Using IEEEtran.cls for IEEE Journals}


\maketitle

\begin{abstract}
Temporal Graph Neural Networks (TGNNs) are pivotal in processing dynamic graphs. However, existing TGNNs primarily target one-time predictions for a given temporal span, whereas many practical applications require { continuous predictions}, that predictions are issued frequently over time. 
Directly adapting existing TGNNs to continuous-prediction scenarios introduces either significant computational overhead or prediction quality issues especially for large graphs. 
This paper revisits the challenge of { continuous predictions} in TGNNs, and introduces {\sc Coden}, a TGNN model designed for efficient and effective learning on dynamic graphs.
{\sc Coden}  innovatively overcomes the key complexity bottleneck in existing TGNNs while preserving comparable predictive accuracy.  {Moreover, we further provide theoretical analyses that substantiate the effectiveness and efficiency of {\sc Coden}, and clarify its duality relationship with both RNN-based and attention-based models. }
Our evaluations across five dynamic datasets show that {\sc Coden} surpasses existing performance benchmarks in both efficiency and effectiveness, establishing it as a superior solution for continuous prediction in evolving graph environments.
\end{abstract}


\vspace{-6mm}
\section{Introduction}
\IEEEPARstart{G}{raphs}, fundamental structures consisting of nodes and edges, are pervasive in representing complex relationships and interactions in numerous domains. In real life, graphs model various scenarios such as social networks where nodes represent individuals and edges symbolize social connections \cite{spikenet, quota, hu2023detecting, ju2019generating}, transportation networks with intersections as nodes and roads as edges \cite{wu2019graph, jin2023spatio,zheng2020gman}, and even Internet networks mapping the vast array of interconnected devices and pathways \cite{DBLP:journals/iotj/ChenZYZL22, khanfor2020graph}.  By leveraging node features and structural information, Graph Neural Networks (GNNs) foster the success of representation learning in non-Euclidean graph data and have become the cornerstone behind a variety of high-impact applications \cite{li2021towards, wu2023graph,reiser2022graph,lu2020spatiotemporal}.

Many modern services are intrinsically built on temporal graphs due to the fact that real-world graphs frequently receive updates including the alteration of nodes, edges, and attributes \cite{mo2022single}. For instance, social networks constantly evolve with user interactions \cite{mislove2007measurement,zhao2016link}, financial networks update with market fluctuations \cite{song2023towards,xu2023promoting}, and communication networks adjust to varying traffic loads \cite{shao2022decoupled,shao2022spatial}. In these domains, Temporal Graph Neural Networks (TGNNs) are prominent methods for capturing the temporal information in the evolving process of graphs. Inspired by representative algorithms ranging from recurrent neural networks (RNNs) \cite{seo2018structured,hajiramezanali2019variational}, attention mechanism \cite{zhang2023dynamic,xu2020inductive}, random walk \cite{nguyen2018dynamic, wang2021inductive} to temporal point process \cite{lu2019temporal, jin2023spatio}, a plethora of TGNN frameworks have been proposed to learn topological and temporal patterns simultaneously and improve the prediction performance. 

\textbf{The problem: continuous prediction in dynamic scenarios.} Despite the success of TGNNs over the past few years, most of the existing approaches only focus on the one-time prediction of a single time step (e.g., the final snapshot). Unfortunately, the frequent updates of graphs can result in rapidly outdated information, compelling users to continually request updated answers from the system. We take the e-commerce scenario as an application, where the historical interactions between users and goods are used to make a prediction on the gross merchandise value \cite{ye2022gaia}. Given the continuous queries on trading platforms (e.g., Amazon, Alibaba Taobao), it is crucial to continuously collect updated historical interactions in order to make accurate predictions for these queries. 
Moreover, many social platforms such as LinkedIn experience millions of interactions between users per day \cite{borisyuk2024lignn}, where the rapidly evolving social relationships also require us to continuously refresh the predicted recommendation results. However, in these continuously evolving scenarios, existing research on TGNNs, originally designed to enable a complicated analysis with a perspective of the global time span, become largely inapplicable \cite{feng2024comprehensive, skarding2021foundations}. 


\textbf{Limitations of existing works.} 
To enable continuous prediction in TGNNs, there are three primary temporal message passing paradigms, namely, \textit{single-snapshot} GNN methods, \textit{RNN-based} methods and \textit{attention-based} methods. The \textit{single-snapshot} GNN methods \cite{zheng2022instant, guo2022subset} simply utilize the embeddings of the current graph to answer the prediction request from the users. As the only paradigm to achieve high efficiency for continuous predictions, single-snapshot GNNs neglect the knowledge learned in previous snapshots while learning new patterns, resulting in degraded model performance in evolving scenarios. The \textit{RNN-based} methods \cite{evolvegcn,MPNN} can achieve the rapid iteration of node representations by collecting the current embeddings with the previous hidden states to formulate the updated hidden states. Due to the inherent {\it forgetting} mechanism of RNNs \cite{graves2012long}, these methods have a limited memory horizon and struggle to effectively model long-term interactions. The \textit{attention-based} methods \cite{sankar2020dysat, ASTGCN} employ the commonly used self-attention mechanism \cite{vaswani2017attention} to compute the attention score between previous node embeddings with current hidden state and selectively retain the significant part for predictions. However, as the mass of historical interaction explodes, attention-based methods incur $O(T^2)$ time complexity given $T$ time steps to compute the attention scores, which unavoidably degrades the prediction efficiency. We demonstrate the characteristics and the complexity of each paradigm in Figure \ref{fig:cat comparison} and Tab. \ref{table:method_comparison}, respectively. 

This work aims to reduce the computational complexity of TGNNs in continuously evolving environments while preserving long-term graph interactions. We note that achieving this goal for current TGNNs listed above is a non-trivial task, especially when giving consideration to both accuracy and efficiency in the context of frequently updated, large-scale graphs. 
First, the mainstream message-passing units of TGNNs focusing on one-time predictions struggle with an imbalanced trade-off between accuracy and efficiency.
Adapting existing methods typically requires either abandoning the previous states (e.g., single-snapshot and RNN-based methods) or incorporating additional graph snapshots into the analysis pipeline (e.g., attention-based methods), which will degrade the prediction performance or increase the computational complexity of the model, respectively. Second, directly adopting the state-of-the-art models for answering the graph updates will lead to significant efficiency challenges when the scale of graphs increases. This inefficiency stems from the tight coupling of the graph encoder (e.g., GCN \cite{DBLP:conf/iclr/KipfW17}, GraphSage \cite{graphsage})  with temporal-processing units (e.g., RNN \cite{cho2014learning}, GRU \cite{chung2014empirical}, LSTM \cite{hochreiter1997long}, Transformers \cite{vaswani2017attention}), requiring frequent regeneration of node embeddings. Thus, a natural question arises: \textit{is there a model that can integrate the advantages of listed paradigms in Figure \ref{fig:cat comparison} and take into account both accuracy and efficiency in evolving scenarios? }

These questions motivate us to design {\sc Coden}, a scalable framework targeting the continuous prediction on TGNNs. 
When handling frequent updates in dynamic scenarios, {\sc Coden} distinguishes itself among mainstream TGNNs by balancing both effectiveness and efficiency. On the theoretical side, {\sc Coden} achieves a complexity for updating node states that matches the minimal complexity of leading TGNNs, as demonstrated in Tab. \ref{table:method_comparison}.  {Furthermore, we establish an RNN--Attention duality for {\sc Coden}: under simple constraints it reduces to a gated recurrent update, while in the general case its unrolled dynamics are equivalent to masked kernel attention. This unified view provides a principled explanation of why {\sc Coden} simultaneously achieves high efficiency and strong predictive performance.} Compared with the state-of-the-art
method TGL~\cite{tgl} which use 98.3 hours to finish the continuous prediction on billion-scale graph \textit{Papers100M} (111M nodes, 1.6B edges) averagely, our method {{\sc Coden} finishes in 3.3 hours on a single GPU. Our contributions are summarized as follows.

$\bullet$ \textbf{New Efficient Model.} 
We propose {\sc Coden}, an efficient and scalable model designed to compute the temporal state of nodes and enable continuous prediction in real-world dynamic scenarios.
%
{\sc Coden} incorporates a series of techniques that facilitate the efficient update of node states in evolving graphs. Additionally, we
prove that our model can achieve the information compression of historical graphs, enabling a superior trade-off between accuracy and efficiency than state-of-the-art algorithms.



 {$\bullet$ \textbf{New Theory.} We reveal an RNN–Attention duality in our temporal-processing paradigm: the state update admits both a gated recurrent interpretation and a kernel attention form, clarifying that its information compression operates by selectively masking redundant signals along the temporal evolution. We also propose an alternative baseline model based on the attention mechanism to further highlight the efficiency and representation quality of {\sc Coden}.}

$\bullet$ \textbf{Empirical Validation.} We conduct extensive experiments on a diverse range of real-world datasets including large-scale graphs to conclusively demonstrate the efficiency and effectiveness of our approach in practical dynamic scenarios.
Compared with state-of-the-art methods, {\sc Coden} can achieve up to $44.80\times$ training acceleration while retaining comparable or better prediction accuracy.

\begin{figure}[t]
    \centering
  \includegraphics[width=3.4in]{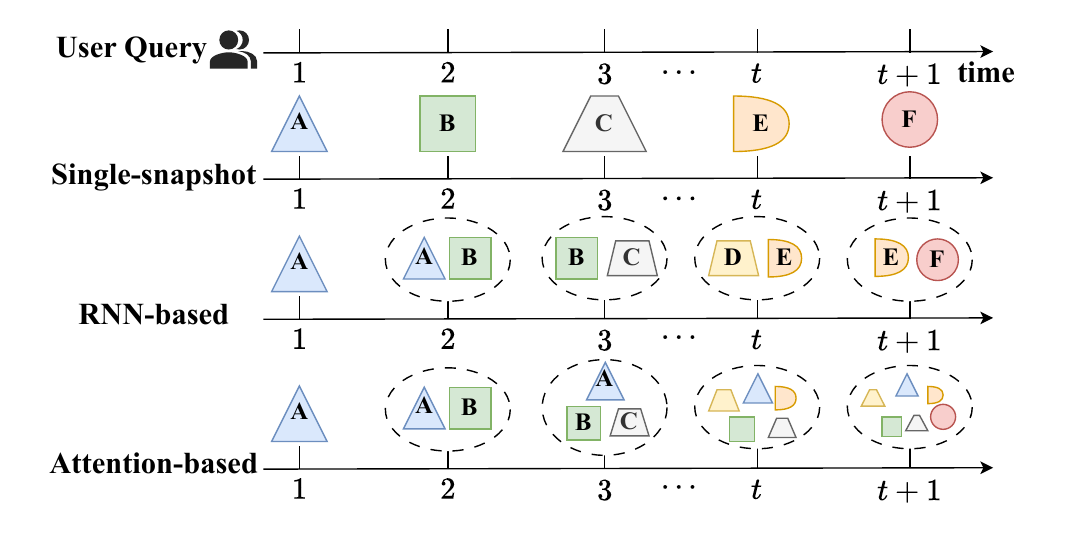}
   
    \caption{ {The comparison between different TGNN paradigms, where each color denotes a different status of the graph.  }}
     
    \label{fig:cat comparison}
    \vspace{-5mm}
\end{figure}

\setlength{\tabcolsep}{0.5mm}{
\begin{table*}[ht]
\centering
\small
\caption{{Comparison of different categories in TGNNs. }
$K$ stands for the number of convolution layers, $p$ means the number edge updates, $\lambda$ is the embedding distance parameter, and $F$ is the dimension of node features. We set the graph propagation error as $\epsilon = \frac{1}{n^{(t)}}$ and ${\bm x}_{max}$ is defined as: $\{{\bm x}_{max}\}_i = \max_{1\leq j\leq F}|{\bm X}^{(t)}_{ij}|$. The introduction of the listed methods can be found in Appendix \ref{app:listed baseline}.} 
 
\begin{tabular}[t]{c|c|c|c|c|c} \toprule[1pt]
\multirow{2}{*}{\bf Categories} & \multirow{2}{*}{\bf Methods} &\multicolumn{2}{c}{\makecell[c]{\bf Minimal Complexity for Updating $p$ Edges}}& {\multirow{2}{*}{{\bf Efficiency}}} &\multirow{2}{*}{\bf Memory Ability}  \\ 
 & &   {\bf Update Embeddings} & {\bf Update States}& &\\ \midrule[0.5pt]
 Single-snapshot&Instant \cite{zheng2022instant}, DynAnom \cite{guo2022subset}, IDOL \cite{IDOL} &$O\left(p\sum_{i=1}^F{||{\bm x}_i^{(t)}||_1}\right)$& N.A.& \MobileSignal{5}& \MobileSignal{1}\\ \midrule[0.5pt]
{RNN-based} & \makecell[c]{TGN \cite{tgn}, TGCN \cite{tgcn}, EvolveGCN \cite{evolvegcn}, \\MPNN \cite{MPNN}, ROLAND \cite{you2022roland}} & $O(Km^{(t)}F)$&$O(pn^{(t)}F^2)$ & \MobileSignal{4} &\MobileSignal{2} \\\midrule[0.5pt]

{Attention-based} & \makecell[c]{DySat \cite{sankar2020dysat}, TGAT \cite{xu2020inductive}, ASTGCN \cite{ASTGCN}, \\DNNTSP \cite{yu2020predicting}, SEIGN\cite{qin2023seign}, DyGFormer \cite{yu2023towards}}  & $O(Km^{(t)}F)$&$O(p\left(n^{(t)}\right)^2F)$ &\MobileSignal{1}&\MobileSignal{5}\\\midrule[0.5pt]

SSM-based  & Coden &  $O\left(p\sum_{i=1}^F{||{\bm x}_i^{(t)}||_1}\right)$&$O(\frac{||{\bm x}_{max}||_1+n^{(t)}}{\lambda}pF^2)$& \MobileSignal{4}&\MobileSignal{4} \\ 
\bottomrule[1pt]
\end{tabular}
 
\label{table:method_comparison}
\hfill
 
\end{table*}
}
\section{Preliminary and Related Work}
In this section, we start by reviewing the continuous-time dynamic graph (CTDG) settings. We then revisit the existing TGNN literature, providing a detailed discussion of these approaches. Specifically, we highlight the limitations of these methods in scenarios that require continuous prediction.

\subsection{Notations}

Consider a directed and attributed graph at time $t$ ($t\ge 0$), denoted as $\mathcal{G}^{(t)} = (\mathcal{V}^{(t)},\mathcal{E}^{(t)},{\bm X}^{(t)})$. Here $\mathcal{V}^{(t)} =\{v_1,v_2,...,v_{n^{(t)}}\}$ represents the set of $n^{(t)}$ nodes, $\mathcal{E}^{(t)}$ constitutes the set of $m^{(t)}$ edges and ${\bm X}^{(t)} = \{{\bm x}_1^{(t)},{\bm x}_2^{(t)},...,{\bm x}_F^{(t)}\}$ is the set of node attribute matrix with ${\bm x}_i^{(t)} \in \mathbb{R}^{n^{(t)}\times 1}$ representing the attribute vector in $i$-th dimension. 
Here we denote ${\bm A}^{(t)}$ and ${\bm D}^{(t)}$ as the adjacency matrix and the degree matrix at time $t$ respectively.
For each node $v \in \mathcal{V}^{(t)}$, $\mathcal{N}_{out}^{(t)}(v)$ denotes the out-neighbors of $v$ and $\mathcal{N}_{in}^{(t)}(v)$ denotes the in-neighbors at time $t$. 
In the setting of a continuous-time dynamic graph (CTDG) \cite{zheng2022instant,li2023zebra}, the graph is treated as an initial graph $\mathcal{G}^{(0)} = (\mathcal{V}^{(0)},\mathcal{E}^{(0)})$ and the subsequent update events consisting of edge insertions and deletions. \footnote{The insertion and deletion of vertices can be managed by adding or removing their respective incident edges. Therefore, this paper focuses exclusively on edge updates and we have $n^{(t)} = n$.} We denote the set of update events as $\Gamma  = \{e_1,e_2,...,e_{t},...\}$, where the edge update $e_{t} = \{u_{t},v_{t}, t\}$ indicates that the link between nodes $u_t$ and $v_t$ will be toggled at time $t$—added if absent, or deleted if present—transforming graph $\mathcal{G}^{(t-1)}$ into $\mathcal{G}^{(t)}$. We leave the notation table in Table \ref{table:fre_notations}.

\begin{table}[htb]
\centering
 \setlength{\abovecaptionskip}{0.5mm}
 	\setlength{\belowcaptionskip}{0.5mm}
\caption{\small Frequently used notations in this paper.}
\label{table:fre_notations}

\renewcommand{\arraystretch}{1.2}
\small

\begin{tabular}{|p{.135\textwidth}| p{.305\textwidth}|} \toprule[1pt]
\textbf{Notations} & \textbf{Descriptions}\\\hline \hline
$\mathcal{G}^{(t)}$& Directed graph at time $t$\\\hline
$n^{(t)}, m^{(t)}$& Numbers of nodes and edges at time $t$\\\hline
$F$& Dimension of the node attribute\\\hline
${\bm X}^{(t)}$&The node attribute matrix and ${\bm X}^{(t)}\in \mathbb{R}^{n^{(t)}\times F}$
\\\hline
${\bm x}_i^{(t)}$&The attribute vector of the $i$-th dimension\\\hline
$\mathcal{N}_{out}^{(t)}(v), \mathcal{N}_{in}^{(t)}(v)$& The out-neighbors and in-neighbors of $v$\\\hline
$e_{t} = \{u_{t},v_{t}\}, \Gamma$& The update event and the event set\\\hline
${\bm r}^{(t)}$& Residue vector at time $t$ and ${\bm r}^{(t)}\in \mathbb{R}^{n^{(t)}\times 1}$\\\hline
${\bm h}^{(t)}$& Reserve vector at time $t$ and ${\bm h}^{(t)}\in \mathbb{R}^{n^{(t)}\times 1}$\\\hline
${\bm H}^{(t)}$&The node embedding matrix at time $t$ \\\hline
${\bm M}^{(t)}$&The state matrix at time $t$ \\\hline
$\bar{\mathcal{A}} , \bar{\mathcal{B}}, \mathcal{C}$&The network parameters in SSM \\\hline
$\alpha$& Teleport probability of random walks \\\hline
$\lambda$& Threshold in to control the embedding distance
\\\bottomrule[1pt]
\end{tabular}
\end{table}

\vspace{1mm}
\noindent
\textbf{Problem Definition.} {Given batches of update events \( \{\Gamma_1, \Gamma_2, ..., \Gamma_k, ...\} \), where each \( \Gamma_k \) ends at time \( t_k \), our objective is to efficiently learn the node state \( {\bm M}^{(t)} \) for \( t \in T = \{t_1, t_2, ..., t_k, ...\} \) and continuously conduct the downsteam tasks with the current graph \( \mathcal{G}^{(t)} \).
 The state ${\bm M}^{(t)}$ should not only represent the current structure of the graph at time $t$, but also capture the historical evolution of each node by integrating dynamic changes over time.} This evolving state is crucial for effectively encoding long-term temporal dependencies and can be leveraged in downstream tasks such as temporal node classification. Our approach aims to achieve this  with  {limited computational resources}, ensuring efficiency as the large-scale graph continuously evolves.



\subsection{General Message Passing with CTDG}
Unlike static graphs, the nodes in temporal graphs can have varying neighbors at different time steps. A commonly used intuition \cite{tgl, tgn,apan, dgnn} to deploy this characteristic is maintaining the node states ${\bm M}^{(t)}(u)\in \mathbb{R}^{ F'}$ ($s\in \mathcal{V}^{(t)}$) to summarize the historical information of evolving neighbors until timestep $t$, where $F'$ is the hidden dimension of the representation. Specifically, given a new edge connecting $e_{t+1} = \{u,v\}$ from node $u$ to node $v$, these algorithms update the state of node $u$ as:
{\small
\begin{align}\label{equ:tgnn}
    {\bm M}^{(t+1)}(u) = \text{UPDT}\left({\bm M}^{(t)}(u), \text{MSG} \left({\mathcal{G}}^{(t)}, {\bm M}^{(t)}(u),  e_{t+1}\right)\right)
\end{align}
}
where $\text{MSG}(\cdot)$ is the message functions including the graph convolution with learnable parameters (e.g. MLPs) and $\text{UPDT}(\cdot)$ represents the temporal-processing units for state updating. In accordance with this framework, the state of node $v$ as well as those of its neighboring nodes $u$ and $v$ will also be updated. Next, we will introduce existing categories of TGNNs based on how they establish dependencies between different time steps using $\text{MSG}(\cdot)$ and $\text{UPDT}(\cdot)$ functions. Each paradigm offers its unique trade-off between accuracy and efficiency, which is demonstrated in Tab. \ref{table:method_comparison}.

\subsection{Single-snapshot Methods}

The single-snapshot GNN methods directly utilize the node embedding as the node states, which will be rapidly updated based on the new interactions. Specifically, single-snapshot GNN methods aim to transmit the updates into node embeddings ${\bm H}^{(t+1)}$ from ${\bm H}^{(t)}$ with the minimal time cost following: 
\begin{align}
    {\bm H}^{(t+1)} = \text{MSG} \left({\mathcal{G}}^{(t)}, {\bm H}^{(t)},  e_{t+1}\right).
\end{align}
To facilitate the incremental updating based on these pre-computed node embeddings, Instant \cite{zheng2022instant}, DynAnom \cite{guo2022subset} and IDOL \cite{IDOL} explore the invariant rules of graph propagation and conduct an invariant-based
algorithm of Personalized PageRank (PPR) to refresh the node embedding locally. The dominant update complexity of these methods can be bounded as $O(p\sum_{i=1}^F\frac{||{\bm x}_i^{(t)}||_1}{\epsilon n^{(t)}})$ given $p$ updates and $\epsilon$ approximation error \cite{zheng2022instant}.
However, only focusing on the current snapshot will miss significant interactions of past time steps, leading to sub-optimal prediction results. 


\subsection{RNN-based Methods} 
RNN-based methods are generally based on the classical RNN \cite{cho2014learning}, GRU \cite{chung2014empirical}, LSTM \cite{hochreiter1997long} algorithm, etc., which simply use both the current input and the previous hidden state to iteratively capture temporal dependencies. For example, the GRU algorithm updates the node states at time $t+1$ as:
{\small
\begin{align*}
&  {\bm Z}^{(t+1)}=\operatorname{sigmoid}\left({\bm W}_Z  {\bm H}^{(t+1)}+{\bm U}_Z  {\bm M}^{(t)}+{\bm B}_Z\right) \\
& {\bm R}_t=\operatorname{sigmoid}\left({\bm W}_R {\bm H}^{(t+1)}+{\bm U}_R {\bm M}^{(t)}+{\bm B}_R\right) \\
& \widetilde{{\bm H}}^{(t+1)}=\tanh \left({\bm W}_H {\bm H}^{(t+1)}+{\bm U}_H\left({\bm R}_t \circ {\bm M}^{(t)}\right)+{\bm B}_H\right) \\
& {\bm M}^{(t+1)}=\left(1- {\bm Z}^{(t+1)}\right) \circ {\bm M}^{(t)}+ {\bm Z}^{(t+1)} \circ \widetilde{{\bm H}}^{(t+1)},
\end{align*}}
where ${\bm W}, {\bm U}, {\bm B} $ denote the trainable parameters of the linear layer and ${\bm H}^{(t+1)}$ can be updated using the $\text{MSG}(\cdot)$. Specifically,
TGCN \cite{tgcn} and EvolveGCN \cite{evolvegcn} and incorporate the Graph Convolutional Network (GCN)\cite{DBLP:conf/iclr/KipfW17} as the $\text{MSG}(\cdot)$ function to regenerate the node embeddings while coupling with an RNN-based module to learn temporal node representations. Similarly, MPNN \cite{MPNN} transforms the node embeddings at different time steps into an RNN module and then captures the long-range dependency in the final hidden state. This category of method generally requires $O(Km^{(t)}F)$ and $O(pn^{(t)}F^2)$ to update the node embeddings and states \footnote{For a clear presentation, we assume the dimension of node state $F' = F$}, respectively. Due to their iterative structure, RNN-based methods can efficiently update node states. However, the simplistic recurrence mechanism often leads to difficulties in retaining historical information, especially as the graph size and temporal scope expand \cite{graphssm, gers2000learning}.

\subsection{Attention-based Methods} To address the forgetting problem of RNNs, attention-based methods rely on the self-attention mechanism and abstain from using recurrence form, which encodes the position of sequences and enables the efficient information flow from past to current representations. We take the representative work APAN \cite{wang2021apan} as an example to demonstrate the core mechanism of this category.
Considering matrices ${\bm Q}\in \mathbb{R}^{n\times F}$ denoted as "query", ${\bm K}\in \mathbb{R}^{n\times F}$ denoted as "keys", and ${\bm V}^{n\times F}$ denoted as "values", the classical self-attention algorithms perform the following computation to obtain the optimized embeddings:
\begin{align*}
    &{\bm M}^{(t+1)} = \mathrm{Attn}({\bm Q},{\bm K},{\bm V})=\mathrm{softmax}\left(\frac{{\bm Q \bm K}^{\top}}{\sqrt{F}}\right){\bm V},\\
    &{\bm Q} = {\bm H}^{(t+1)} {\bm W}_q,  {\bm K} = {\bm M}^{(t)} {\bm W}_k,
    {\bm V} = {\bm M}^{(t)} {\bm W}_v,
\end{align*}
where ${\bm W}_q, {\bm W}_k, {\bm W}_ \in \mathbb{R}^{ F\times F'}$ are the network parameters. The dot-product term $\left(\frac{{\bm Q \bm K}^{\top}}{\sqrt{F}}\right)$ takes the role of weighting the interactions between entity "query-key" pairs. A higher value within this term increases the contribution of ${\bm V}$ to the embedding space. Thus, attention-based methods create the expressive attention score to capture the relationship between the current embedding and the state of the last time step. Following this intuition, DySat \cite{sankar2020dysat} employs the generalized GAT module \cite{velivckovic2018graph} to integrate the embeddings of a single node from different time steps to generate its refreshed one. ASTGCN \cite{ASTGCN} and DNNTSP \cite{yu2020predicting} further incorporate the self-attention mechanism to capture the spatial and temporal dependency for enhanced representation quality. For each time step, these methods generally require $O(T\left(n^{(t)}\right)^2F)$ time complexity to calculate the final output given $T$ time step.
While attention mechanisms can retain the most relevant parts of the sequences to avoid the forgetting issue, they can become computationally expensive with frequent updates \cite{thomasgraph}.

\textbf{Other methods.} (i) \textit{SNN-based methods.} Another typical mechanism of this category is based on the biological Spiking Neural Networks (SNNs), which simulate the brain behaviors and maintain the membrane potential given the data sequences. SpikeNet \cite{spikenet} retrieves the node embeddings from multiple time steps and finally generates the prediction results by the spike firing 
process. Dy-SIGN \cite{yin2024dynamic} incorporates SNNs mechanism into dynamic graphs to mitigate the information loss and memory consumption problem. Nevertheless, the demand for multiple simulation steps to generate reliable embeddings can significantly degrade the efficiency of these TGNN methods. (ii) \textit{SSM-based methods.} There are also some works which employs the SSM mechanism on temporal graphs.
For example, STG-Mamba \cite{li2024stg} formulates the feature stream of each node as the long-term context, which improves the embedding quality for feature-varying graphs. Graph-SSM \cite{graphssm} addresses the unobserved graph mutations between consecutive snapshots, and achieves an effective discretization with long-term information. However, these methods only focus on the discrete-time dynamic graph and fail
to model the continuous topology changes as the graph evolves. Furthermore, directly adapting these methods will incur significant computational overhead as the graph evolves, creating a gap between current algorithms and continuous prediction in practical dynamic scenarios. 

\subsection{State Space Models}
In recent years, structured State Space Models (S4) \cite{guefficiently, gu2021combining} have emerged as promising frameworks for modeling long-distance sequences, offering the advantage of only a linear increase in computational cost.
Given the input series $x(t) \in \mathbb{R}$, structured State Space Models (S4) \cite{guefficiently, gu2021combining} formulates the state variable series $h(t)\in \mathbb{R}$ and the output series $y(t)\in \mathbb{R}$ using the following equations:

{\footnotesize
\hspace{-0.03\textwidth}
\begin{minipage}{0.25\textwidth}
\begin{align}\label{equ:cont}
\left\{\begin{array}{cl}
h'(t) &= \mathcal{A}h(t) + \mathcal{B}x(t) \\
y(t) &= \mathcal{C}h(t), \end{array} \right.
\end{align}
\end{minipage}
\hspace{-0.01\textwidth}
\begin{minipage}{0.23\textwidth}
\begin{align}\label{equ:dis}
\left\{\begin{array}{cl}
h_t &= \bar{\mathcal{A}}h_{t-1} + \bar{\mathcal{B}}x_t\\
y_t &= \mathcal{C}h_t,
\end{array} \right .
\end{align}
\end{minipage}
}
where $(\mathcal{A}, \mathcal{B}, \mathcal{C}, \bar{\mathcal{A}}, \bar{\mathcal{B}})$ are trainable parameters.
Here Equ. \ref{equ:cont} and \ref{equ:dis} are the continuous and discretization process, respectively. Compared with the RNN-based or attention-based mechanism, S4's strength lies in its adherence to a linear mechanism, which guarantees enhanced stability control \cite{graphssm}. This facilitates effective long-term modeling of sequences through meticulous initialization of state space layer parameters \cite{gu2020hippo,orvieto2023resurrecting}. As a result, we employ the SSMs as our temporal-processing unit to meticulously balance the accuracy and efficiency of our {\sc Coden} framework.

\section{Methodology}

\subsection{Overview}
The goal of {\sc Coden} is to minimize the cost of learning the new node state ${\bm M}^{(t)}$ across a series of time steps. The realization of this goal relies on two key components: (a) a scalable algorithm to enable the aggregation of new neighborhood information (reflected in $\text{MSG}(\cdot)$); and (b) an efficient paradigm to enable the summarization of evolution information (reflected in $\text{UPDT}(\cdot)$).
For component (a), state-of-the-art methods such as DyGFormer \cite{yu2023towards}, EvolveGCN \cite{evolvegcn}, and ROLAND \cite{you2022roland} rely heavily on a large number of trainable parameters for neighborhood aggregation. Frequent updates necessitate continuous re-training of these parameters, resulting in significant computational overhead. For component (b), dynamic scenarios require the node states to  continuously summarize the evolving information. 
However, as highlighted in Tab. \ref{table:method_comparison}, existing methods struggle to balance accuracy and efficiency, leading to suboptimal performance in such cases.

To address these challenges, {\sc Coden} introduces a novel approach based on two key principles. First, it decouples graph propagation from the training process and incrementally integrates permutation information from $p$ updates to transition the node embedding from ${\bm H}^{(t)}$ to ${\bm H}^{(t+p)}$. This embedding is parameter-free, enabling scalable and efficient updates to the node state ${\bm M}^{(t+p)}$. Second, {\sc Coden} employs a lazy-sampling strategy to discretize the continuous update sequence, effectively managing the trade-off between accuracy and efficiency during the summarization of evolving information. We illustrate the design of {\sc Coden} with following four steps:

$\bullet$ \textbf{Step 1 - Incremental state update.} We simplify the mechanisms of existing TGNNs by incrementally integrating the SSM structure with parameter-free node embeddings, enabling a efficient refresh of evolving node-state information.

$\bullet$ \textbf{Step 2 - Efficient batch processing.} For efficient batch processing, we constrain the potential disruptions from updates on node embeddings and performs a lazy-sampling method to extract intermediate embeddings. Our theoretical analysis shows that the refreshed node states effectively compress historical graph information, backed by an approximation guarantee.

$\bullet$ \textbf{Step 3 - RNN-Attention duality.}
We formalize the temporal-processing mechanism of {\sc Coden} by showing that, under mild parameter constraints, its state-space update reduces to a gated recurrent form, while in the general case its unrolled dynamics are equivalent to masked kernel attention. This duality explains how {\sc Coden} compresses historical information by selectively masking redundant temporal signals.

$\bullet$ \textbf{Step 4 - Comparison with attention mechanism.} Building on the previous equivalence proof, we propose an alternative model that leverages the attention mechanism to enable a fair comparison. Theoretical analysis demonstrates that our model outperforms the alternative in both effectiveness and efficiency.

To this end, we demonstrate an efficient method for incrementally updating node embeddings while ensuring accuracy, which is strictly required by our theoretical analysis. The overall framework of Coden is illustrated in Fig. \ref{fig:overview}.
For the limitation of space, all missing proofs and the auxiliary materials can be found in the appendix.

\subsection{ {Efficient State Update}}\label{sec:ssm}
\textbf{PPR-based Node Embeddings.}
Existing RNN-based and attention-based methods are inclined to adopt the classical learnable encoder (e.g., GCN \cite{DBLP:conf/iclr/KipfW17}, GraphSage \cite{graphsage}, GAT \cite{velivckovic2018graph}) to generate node embeddings, which require costly retraining whenever the graph changes. To address this issue, we adopt the Personalized PageRank (PPR) based embedding, because it is more scalable in computation \cite{scara,luo2012disks, luo2019baton, IDOL} and it can also alleviate over-smoothing issues \cite{wu2022non}. Specifically, after setting ${\bm P}^{(t)} = {{\bm A}^{(t)}}^{\top}{{\bm D}^{(t)}}^{-1}$, we defined our PPR-based node embedding as following:
{\footnotesize
 
\begin{align*}
 {\bm Z}^{(t)} =\sum_{l=0}^{\infty} \alpha(1-\alpha)^l \left({\bm P}^{(t)}\right)^l  {\bm X}^{(t)} = \alpha \left(I-(1-\alpha){\bm P}^{(t)}\right)^{-1} {\bm X}^{(t)},
\end{align*}
}
 
where ${\bm P}^{(t)}$ is the normalized adjacency matrix at time $t$ and $\alpha$ is the teleport probability of random walks. To improve computational efficiency without sacrificing model effectiveness, we adopt the approximation ${\bm H}^{(t)}$ with the guarantee $||{\bm H}^{(t)} - {\bm Z}^{(t)}||_1 \leq n^{(t)}\epsilon$. For the clearness of presentation, we leave the basic calculation of our embedding and the incremental update algorithm in Appendix \ref{app:forward} and Sec. \ref{sec: incre_update} respectively.

\vspace{1mm}
\noindent
\textbf{From Embeddings to States.}
After obtaining the parameter-free node embeddings ${\bm {H}}^{(t)}$ at different timestamps, it is necessary to integrate the evolution information from these embedding sequences into the node states ${\bm {M}}^{(t)}$. However, existing mechanisms, such as RNNs and attention-based methods, often face challenges in balancing accuracy and efficiency, as highlighted in Tab. \ref{table:method_comparison}. Inspired by recent advancements in applying physical systems to sequence processing \cite{mamba}, we leverage the classical State Space Models (SSMs) \cite{kalman1960new} to effectively summarize rich temporal information across different time steps. Specifically, given an edge update $e_{t+1}$ at time $t+1$, we update the state matrix ${\bm M}^{(t+1)}(u)$ for each $u\in \mathcal{V}^{(t+1)}$ with the following equations:
{
\small
\begin{equation}\label{equ:Coden_all}
    {\bm M}^{(t+1)}  = \bar{\mathcal{A}} \cdot {\bm M}^{(t)} + \bar{\mathcal{B}} \cdot {\bm H}^{(t+1)},
\end{equation}
}
where $\bar{\mathcal{A}}, \bar{\mathcal{B}} \in \mathbb{R}^{F'\times F}$ are trainable parameters adhering to the standard formulations utilized in the literature \cite{mamba,guefficiently}. The prediction of time $t+1$ is then generated through an MLP layer as ${\bm Y}^{(t+1)} = \text{MLP}( {\bm M}^{(t+1}) $. Compared with the methods \cite{graphssm, li2024stg} applying SSMs on temporal graph tasks, our approach is the first to prioritize efficiency of continuous prediction in practical dynamic scenarios. By ensuring the precision of the node embeddings, our method achieves a significant efficiency advantage over current algorithms while maintaining comparable prediction accuracy. In the following sections, we will explore the theoretical effectiveness of {\sc Coden} (see Sec. \ref{sec: info compre} and \ref{sec:equ}) and the underlying rationale behind its efficiency gap (see Sec. \ref{sec:complexity}).

\begin{figure}[t]
    \centering
  \includegraphics[width=3.4in]{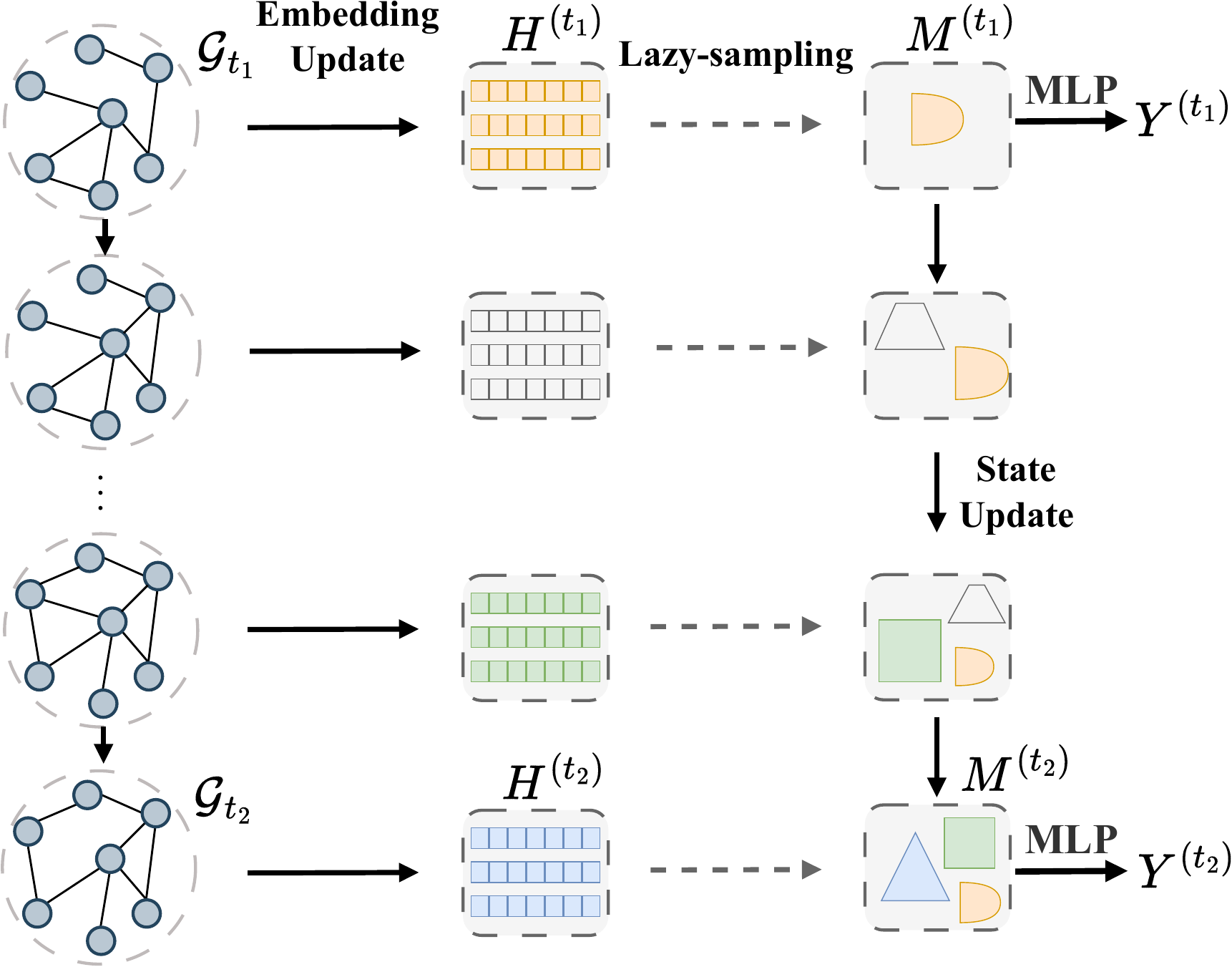}
   
    \caption{ {An illustration of the Coden framework. ${\bm M}^{(t)}$, ${\bm H}^{(t)}$ and ${\bm Y}^{(t)}$ denote the node state matrix, embedding matrix and the prediction results at time $t$, respectively.  }}
     
    \label{fig:overview}
\end{figure}

\subsection{Efficient Batch Processing } \label{sec: info compre}
In our previous discussions, we update the node states iteratively at each time step when new updates arrive, as outlined in Equ. \ref{equ:Coden_all}. However, updates in TGNNs often occur in batches, especially in update-dense dynamic 
scenario. In such cases, processing each update individually is associated with time complexity of 
$O(p)$ for $p$ updates, leading to prohibitive computation overhead. 
To address this challenge, we propose to sample intermediate embeddings across different snapshots. 

\vspace{1mm}
\noindent
\textbf{Lazy-sampling Strategy.} Since each update will only cause a minor change on the embedding space \cite{zheng2022instant, IDOL}, we consider to skip the snapshots until there is a significant shift in the embedding space of ${\bm H}^{(t)}$. However, monitoring the shift of ${\bm H}^{(t)}$ requires sequential updates with each new arrival in the system, which significantly reduces efficiency. Therefore, we provide the following upper bound of the shift between ${\bm H}^{(t)}$ and ${\bm H}^{(t+1)}$ given an update at $t+1$ to estimate the potential shift and reduce the computational overhead:

\begin{lemma}\label{lemma:single shift}
    If there exists an edge update at time $t+1$ and $||{\bm H}^{(t)} - {\bm Z}^{(t)}||_1 \leq n^{(t)}\epsilon$ holds for  time $t$, the difference between ${\bm H}^{(t)}$ and ${\bm H}^{(t+1)}$ satisfies:
    {\footnotesize
    \begin{align}\label{equ:single diff}
      ||{\bm H}^{(t+1)} - {\bm H}^{(t)}||_1 \leq  \frac{1-\alpha}{\alpha}\left\lVert\left( {\bm P}^{(t+1)} - {\bm P}^{(t)}\right)\cdot{\bm x}_{max}\right\rVert_1 +2n^{(t)}\epsilon,
    \end{align}}
    where ${\bm x}_{max}$ is the row-wise maximum absolute value vector and the $i$-th entry of ${\bm x}_{max}$ is defined as: $\{{\bm x}_{max}\}_i = \max_{1\leq j\leq F}|{\bm X}^{(t)}_{ij}|$.
\end{lemma}

Interestingly, the R.H.S. of Equ~\ref{equ:single diff} in Lemma \ref{lemma:single shift} can be calculated efficiently as 
it only involves multiplication between a sparse matrix $\left( {\bm P}^{(t+1)} - {\bm P}^{(t)}\right)$ and vector ${\bm x}_{max}$. 
We set $\epsilon$ to $O(\frac{1}{(n^{(t)})^2})$ to ensure a small uppper bound.
Since $||{\bm H}^{(t+p)} - {\bm H}^{(t)}||_1 \leq \sum_{i=0}^{p-1} ||{\bm H}^{(t+i+1)} - {\bm H}^{(t+i)}||_1$, we can accumulate the upper bound as indicated in Equ. \ref{equ:single diff} to further estimate the shift between ${\bm H}^{(t)}$ and ${\bm H}^{(t+p)}$ given a batch containing $p$ updates.


We incorporate this lazy-sampling strategy in {\sc Coden}, which is presented in Alg. \ref{alg:Coden} starting with $\mathcal{G}^{(0)}$. First, for each prediction time $t_{k+1} \in T$, we collect each update event happening between $t_{k}$ and $t_{k+1}$. Then the upper bound for each edge update will be calculated and accumulated
in $\sigma$ (line \ref{line: accumu}), which aims to delay the sampling of embeddings until the accumulated error exceeds a certain control error $\lambda$. Subsequently, we run the algorithm described in Sec. \ref{sec:ssm} to efficiently update the embeddings (line \ref{line:update}).
Ultimately, we output the sampled embedding matrices and use them as input for the SSM model described in Equ. \ref{equ:Coden_all} (line \ref{line:ssm}). The prediction of time $t_{k+1}$ is then generated through an MLP layer as ${\bm Y}^{(t_{k+1})} = \text{MLP}( {\bm M}^{(t_{k+1})}) $.

\vspace{1mm}
\noindent
\textbf{Information Compression by Lazy-sampling}. As {\sc Coden} involves omitting certain evolving information during the sampling process, it is crucial to assess its influence on the model's performance. By leveraging the upper bound on embedding distance, we effectively control the extent of information loss from the evolving dynamics. Interestingly, 
we observe that {\sc Coden} is equally to compress the evolving information between time $t_{k}$ and $t_{k+1}$ of the graph into its state matrix ${\bm M}^{(t_{k+1})}$, preserving the essential temporal dependencies within the graph. 

To simplify the illustration, we assume that there are two sampled embeddings between the prediction time $t_{k}$ and $t_{k+1}$ and the corresponding sampling time stamps are denoted as $\tau_1$ and $\tau_2$ ($t_{k} \leq \tau_1 \leq \tau_2 \leq t_{k+1}$). 
Then the distance of two adjacent sampled embedding using the lazy-sampling process are bounded such that $||{\bm H}^{(\tau_1)}-{\bm H}^{(\tau_2)}||\leq \lambda$.
Note that there are still multiple edge updates between $\tau_1$ and $\tau_2$, where the intermediate embeddings formed by these updates are omitted to enhance efficiency.
 Specifically, assuming there exist $p_{s}$ edges between two consecutively sampled embeddings ${\bm H}^{(\tau_1)}$ and ${\bm H}^{(\tau_2)}$, we demonstrate that {\sc Coden} can effectively incorporate these 
$p_{s}$ omitted embeddings as inputs to the SSM, as detailed in the following proposition:
\begin{proposition}\label{pro:approx}
    If there are $p_{s}$ edges between the sampling time $\tau_1$ and $\tau_2$ and $||{\bm H}^{(t)} - {\bm Z}^{(t)}||_1 \leq n^{(t)}\epsilon$ holds for each time step $t$, then the input ${\bm H}^{(\tau_2)}$ saved in ${\bm M}^{(\tau_2)}$ can be regarded as the approximation of the normalized summation of all exact PPR-based embeddings between $\tau_1$ and $\tau_2$ such that:
{\small
 
\begin{align*}
    \left|\left|{\bm H}^{(\tau_2)} - \sum_{i=1}^{p_s}\sum_{l=0}^{\infty} \alpha\left((1-\alpha)^l {\bm P}^{(\tau_1+i)}\right)^l {\bm X}^{(\tau_1+i)} {\Lambda}_{i}\right|\right|_1 \leq \lambda, 
\end{align*}
}
where $(\tau_1+i)$ is the time step of the $i$-th edge between $\tau_1$ and $\tau_2$ and $\Lambda_i\in \mathbb{R}^{F\times F}$ is the hidden non-negative diagonal matrices saved in $\bar{\mathcal{A}}$ which satisfy $\sum_{i=1}^{p_s}\Lambda_i={\bm I}$.
\end{proposition}

Based on this proposition, the state update process is equivalent to inputting the normalized summation of the exact PPR-based embeddings between $\tau_1$ and $\tau_2$ into the SSM module. Specifically, the hidden diagonal matrices $\Lambda_i$ ($1\leq i \leq p_s$) will be implicitly modified during the training of the parameter $\bar{\mathcal{A}}$, which will in turn influence the amount of information retained in future time windows \cite{graphssm}. In this scenario, we effectively compress the continuously evolving process, selectively retaining information that enhances prediction capabilities for downstream tasks. In Sec. \ref{sec:equ}, we will further reinforce this key insight by demonstrating its equivalence with attention-based algorithms.

 {
\subsection{Embedding Update Guaranteeing
\texorpdfstring{$\lVert \bm{H}^{(t)} - \bm{Z}^{(t)}\rVert_1 \le n^{(t)}\epsilon$}
{||H(t)-Z(t)||_1 <= n(t) epsilon}}
\label{sec: incre_update}
As mentioned, our objective is to efficiently propagate frequent updates into node embeddings from the preceding time step. The key rationale is that an edge update typically affects only a local neighborhood, so most node states remain nearly unchanged \cite{DBLP:journals/tkde/MoL23}. Motivated by this observation, single-snapshot methods such as Instant \cite{zheng2022instant} and IDOL \cite{IDOL} update embeddings locally on the affected nodes, and we incorporate these local-update strategies into our framework as the additional technique to accelerate state refresh after each update. As a result, our framework maintains the approximation condition in Lemma~\ref{lemma:single shift} and Proposition~\ref{pro:approx}, namely $\lVert {\bm H}^{(t)}-{\bm Z}^{(t)} \rVert_1 \le n^{(t)}\epsilon$ for a given $\epsilon$ at every time step. More details are provided in Appendix~\ref{sec: incremental embed}.
}

\IncMargin{1em}
\setlength{\textfloatsep}{0.5pt}%
\begin{algorithm}[tb]
\small
   \SetKwInOut{Input}{Input}\SetKwInOut{Output}{Output}
   \Input{Initial Graph $\mathcal{G}^{(0)} = (\mathcal{V}^{(0)},\mathcal{E}^{(0)})$, initial embedding ${\bm H}^{(0)}$, update set $\{(u_1,v_1), (u_2,v_2), ...., (u_t,v_t), ...\}$, prediction time $T = \{t_1, t_2, ...., t_k, ...\}$.}
   
   \Output{Prediction result  ${\bm Y}^{(t_{k+1})}$ for $k = 0, 1, 2, 3, ...$}

$\sigma = 0$; $\Gamma = \emptyset$;  $t= 0$; $k= 0$ \; 

\For{$k = 0, 1, 2, 3, ...$}{
\ForEach{update $(u_t,v_t)$ between $t_{k}$ and $t_{k+1}$}{
$\sigma+ = \frac{1-\alpha}{\alpha}\left\lVert\left( {\bm P}^{(t)} - {\bm P}^{(t-1)}\right)\cdot{\bm x}_{max}\right\rVert_1 +2n\epsilon$\; $\Gamma.\text{add}((u_t,v_t))$; $t+=1$\;\label{line: accumu}

\If{$\sigma > \lambda$ or $t=t_{k+1}$}{
\ForEach{$ {\bm h}^{(t-|\Gamma|)}\in {\bm H}^{(t-|\Gamma|)}$ \textbf{in parallel}}{
$\label{line:update}\textit{Embedding Update}(\mathcal{G}^{(t-|\Gamma|)}, \Gamma, {\bm h}^{(t-|\Gamma|)}, $\\$ {\bm P}^{(t-|\Gamma|)}), {\bm P}^{(t)})$;
$\sigma = 0$; $\Gamma = \emptyset$\; 
}

${\bm M}^{(t)} = \bar{\mathcal{A}} \cdot {\bm M}^{(t-|\Gamma|)} + \bar{\mathcal{B}} \cdot {\bm H}^{(t)}$\;\label{line:ssm}}}
${\bm Y}^{(t_{k+1})} = \text{MLP}( {\bm M}^{(t_{k+1})}) $\; 
}

\caption{ {{\sc Coden} on Continous Prediction}}\label{alg:Coden}
\end{algorithm}\DecMargin{1em}
 

\subsection{Complexity Analysis}\label{sec:complexity}

In this section, we analyze the time complexity of {\sc Coden}, which is dominated by two processes: the operation in the initial propagation, the complexity of updating node embeddings, and the complexity of updating node states. We will analyze the complexity under two edge arrival patterns: (i) \textit{the average case:} edges arriving randomly, with each node having an equal probability of being the starting node, and (ii) \textit{the worst case:} 
 edges arriving randomly, but with varying probabilities for each node to be the starting node.

\vspace{1mm}
\noindent
\textbf{Complexity of updating node embeddings.}
Given the attribute vector ${\bm x}^{(0)}$, the complexity for propagating this vector using \textit{Forward Push} with error $\epsilon$ can be bounded by $O(\|{\bm x}^{(0)}\|_1/\alpha\epsilon)$ according to \cite{scara, andersen2006local}. We show our results of complexity analysis as follows:
\begin{lemma}\label{lemma:emb complexity}
    Assuming there exist $p$ edge updates during the prediction time $t_k$ and $t_{k+1}$, the complexity of updating node embeddings from $t_k$ to $t_{k+1}$ is $O\left(\frac{p}{\epsilon n^{(t_k)}}\sum_{i=1}^F{||{\bm x}_i^{(t_k)}||_1}\right)$ under pattern (i)  and $O\left(\frac{p}{\epsilon }\sum_{i=1}^F{\max_{u\in \mathcal{V}^{(t)}}|{\bm x}_i^{(t_k)}(u)|}\right)$ under pattern (ii).
\end{lemma}

\vspace{1mm}
\noindent
\textbf{Complexity of updating node states.} Since the number of sampled embeddings can affect the time of training in SSM, hence we first solve the problem of how many embeddings will be sampled given $p$ update events and the threshold $\lambda$. Then we can compute the number of iterations and the complexity of updating node states. We formally show the result with the following lemma:
 
\begin{lemma}\label{lemma: up state}
    Updating the node states of {\sc Coden} requires $O\left(\frac{pF^2}{\lambda}(||{\bm x}_{max}||_1+\epsilon(n^{(t)})^2)\right)$ time under pattern (i) and $O\left(\frac{pF^2n^{(t)}}{\lambda}{ \max_{u\in \mathcal{V}^{(t)}}(|{\bm x}_{max}(u)|+\epsilon n^{(t)}})\right)$ under pattern (ii), where $\{{\bm x}_{max}\}_i = \max_{1\leq j\leq F}|{\bm X}^{(t)}_{ij}|$.
\end{lemma}

\section{RNN–Attention Duality in {\sc Coden}}

In this section, we provide a theoretical analysis to clarify the mechanism and expressive form of {\sc Coden}. We first show that, under simple constraints on the trainable parameters, {\sc Coden} reduces to a standard gated recurrent update, providing an RNN-style interpretation of its memory evolution. We then analyze the general setting and prove that the unrolled dynamics of {\sc Coden} are equivalent to masked kernel attention, revealing an attention-style weighted aggregation over historical embeddings. Together, these results establish {\sc Coden} as a unified interpretation that connects recurrent gating with attention aggregation within a single model.

\subsection{A Gated-RNN View Under Simple Constraints}\label{sec:rnn}
We first provide an recurrent view that connects our update rule to classical gated RNNs. The key observation is that {\sc Coden} evolves its memory through a state-space recurrence. Therefore, under simple constraints on the trainable parameters, the update reduces to the Gated-RNN interpolation that is widely used in recurrent networks. We formalize this view in the following proposition:

\begin{lemma}[Gated-RNN View under Constrained Parameters]
\label{lem:coden_gru_form}
Assume $(\bar{\mathcal A}^{(t)},\bar{\mathcal B}^{(t)})$ are obtained by ZOH discretization~\cite{mamba,guefficiently}, i.e.,
$\bar{\mathcal A}=\exp(\bm\Delta\mathcal A)$ and
$\bar{\mathcal B}=(\bm\Delta\mathcal A)^{-1}\bigl(\exp(\bm\Delta\mathcal A)-\bm I\bigr)\bm\Delta\mathcal B$,
where $\mathcal A$, $\mathcal B$, and $\bm\Delta$ are trainable parameters.
Under the constraint $\mathcal A=-\bm I$, $\mathcal B=\bm I$, and $\Delta^{(t)}>0$, {\sc Coden} can be rewritten in gated-RNN form:
{\small
\begin{align*}
\bm M^{(t+1)}=\bm z^{(t+1)}\bm M^{(t)}+\bigl(\bm I-\bm z^{(t+1)}\bigr)\bm H^{(t+1)}, \bm z^{(t)}=\exp(-\Delta^{(t)}).
\end{align*}
}
\end{lemma}

This shows that {\sc Coden} can be transformed into an RNN structure without changing its overall modeling pipeline, and it provides a complementary perspective on how the model maintains and updates temporal memory.

\subsection{Equivalence with Kernel Attention}\label{sec:equ}

As pioneers in applying SSM to temporal graphs, we offer a theoretical perspective that demonstrates our algorithm's equivalence to the kernel attention mechanism \cite{katharopoulos2020transformers, peng2021random, choromanski2020rethinking}, an enhanced variant derived from the classical self-attention mechanism \cite{vaswani2017attention}. In comparison to existing approaches utilizing self-attention on temporal graphs, we further theoretically prove that our framework not only achieves greater efficiency but also exhibits enhanced expressive power in its embeddings. 

Given ${\bm M}^{(0)} =\bar{\mathcal{A}}^{(0)} \cdot {\bm 0}+\bar{\mathcal{B}}^{(0)} \cdot {\bm H}^{(0)} = \bar{\mathcal{B}}^{(0)}{\bm H}^{(0)}$ where $\bar{\mathcal{A}}^{(t)}$ and $\bar{\mathcal{B}}^{(t)}$ denote the parameter metrices at time $t$. Following the state update principle \footnote{Without the loss of generalization, here we simplify the notation and assume the sample time corresponds with the prediction time for a clear presentation.} of Equ. \ref{equ:Coden_all}, the state matrix at $t$ can be formulated as:
\begin{footnotesize}
\begin{align}
    {\bm M}^{(t)} = &\bar{\mathcal{A}}^{(t)} \bar{\mathcal{A}}^{(t-1)}...\bar{\mathcal{A}}^{(1)} \bar{\mathcal{B}}^{(0)} {\bm H}^{(0)}+ \bar{\mathcal{A}}^{(t)} \bar{\mathcal{A}}^{(t-1)}... \bar{\mathcal{B}}^{(1)} {\bm H}^{(1)}+...\nonumber\\
    &+ \bar{\mathcal{A}}^{(t)} \bar{\mathcal{A}}^{(t-1)}\bar{\mathcal{B}}^{(t-2)}{\bm H}^{(t-2)}+\bar{\mathcal{A}}^{(t)}\bar{\mathcal{B}}^{(t-1)}{\bm H}^{(t-1)} + \bar{\mathcal{B}}^{(t)}{\bm H}^{(t)} \nonumber\\
    &= \sum_{s=0}^t \prod_{i=s+1}^t \bar{\mathcal{A}}^{(i)} \bar{\mathcal{B}}^{(s)}{\bm H}^{(s)}, 
\end{align}
\end{footnotesize}
where we define $\prod_{i=t+1}^t \bar{\mathcal{A}}^{(i)} = {\bm I}$. When we vectorize the prediction result over time $[0, t]$ and restructure the mathematical expression, we can derive the formula representing our prediction across all time steps:
\begin{align}
    {\bm Y}^{(0:t)} = \left({\bm L}\circ {\bm Q \bm K}^{\top}\right) \cdot {\bm V},
\end{align}
where (i) ${\bm L}$ is a lower-triangular matrix and ${\bm L}_{uv} = \prod_{i=v+1}^u \bar{\mathcal{A}}^{(i)}$; (ii) $\bm Q$ is a lower-triangular matrix and ${\bm Q}_{uv} = \mathcal{C}^{(u)}$; (iii) $\bm K$ is a diagonal matrix and ${\bm K}_{uu} = \bar{\mathcal{B}}^{(u)}$; (iv) ${\bm V}$ is the vectorized node embedding generated at all time steps and ${\bm V}(u) = {\bm H}^{(u)}$. 

As illustrated in Fig. \ref{fig:atte}, we observe that the fundamental methodology of Coden for processing temporal information is fully aligned with the definition of the kernel attention mechanism \cite{choromanski2020rethinking}. which is developed to mitigate the “over attention” phenomenon \cite{DBLP:conf/icml/DaoG24} in computer vision \cite{DBLP:conf/iclr/DarcetOMB24} and natural language processing (NLP) areas \cite{DBLP:conf/iclr/XiaoTCHL24}.
Upon revisiting the motivation behind kernel attention in NLP, where it is crucial to selectively maintain and degrade specific tokens \cite{DBLP:conf/iclr/XiaoTCHL24}, it becomes apparent that the trainable matrices $\bar{\mathcal{A}}^{(i)}$ ($0\leq i \leq t$) play a pivotal role.
Recalling the significant impact of $\bar{\mathcal{A}}^{(i)}$ in proposition \ref{pro:approx}, $\bar{\mathcal{A}}^{(i)}$ enable Coden to selectively neglect the status of certain snapshots and thus effectively controlling how much information is preserved over time. As a result, 
these matrices effectively mask redundant information from the evolving process, thereby providing valuable contexts to the node state and significantly boosting performance in downstream tasks. 

\begin{figure}[t]
    \centering
\includegraphics[width=3.5in]{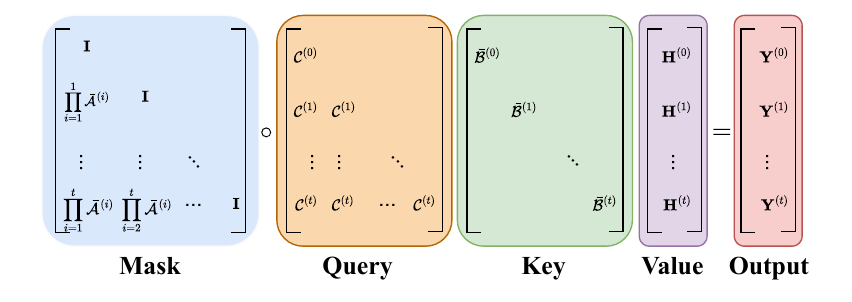}

    \caption{The equivalent kernel attention in Coden.}

    \label{fig:atte}
\end{figure}

\subsection{Superior Embedding with SSM}\label{sec:super}

With the equivalence mentioned earlier, an important question arises: \textit{Given the equivalence, why do we still need the SSM module for {\sc Coden} when processing temporal graphs, especially in contrast to the attention module?} In order to further highlight the merits of incorporating the SSM mechanism into {\sc Coden}, we further provide a matched ablation variant named as {\sc Coden}-A. This variant retains all the features of {\sc Coden} but substitutes the SSM module with the \textit{attention mechanism}, providing a direct comparison of their efficiency and effectiveness. Following the standard structure building of \cite{apan,yu2020predicting}, we compute the temporal dependency of node embeddings over multiple time steps by following the attention expression: 
\begin{footnotesize}
\begin{align*}
   {\bm M}^{(0:t)} =\mathrm{softmax}\left({\left({\bm W_q}{\bm H}^{(0:t)}\right) \cdot\left({\bm W_k}{\bm H}^{(0:t)}\right)}/{\sqrt{F'}}\right)\cdot\left({\bm W_v}{\bm H}^{(0:t)}\right) 
\end{align*}
\end{footnotesize}
where ${\bm W_q}, {\bm W_k}, {\bm W_v} \in \mathbb{R}^{F'\times F}$ are trainable parameters to calculate the queries, keys, and values, facilitating the modeling of interactions among different states across time steps. Next, we will compare Coden and Coden-A with respect to the time complexity and the representation quality:

 (i) \textbf{Time complexity comparison.} Given the prediction time $t$, Coden-A is required to calculate the dependence between the embedding at $t$ and other time steps. Formally, we obtain the state at time $t$ as: $\boldsymbol{M}^{(t)}=\sum_{s=0}^t\mathrm{softmax}\left(\frac{\left(\boldsymbol{W}_q\boldsymbol{H}^{(t)}\right)\cdot\left(\boldsymbol{W}_k\boldsymbol{H}^{(s)}\right)^\top}{\sqrt{F^{\prime}}}\right)\cdot\left(\boldsymbol{W}_v\boldsymbol{H}^{(s)}\right)$. Compared with the information compression mechanism of Coden, Coden-A suffers from a 
        $O(t(n^{(t)})^2 F)$ time complexity when the time window is extended to $t$. 
        
(ii) \textbf{Quality comparison based on Dirichlet Energy.} As an emerging measurement to access the quality of node representations, the degree of over-smoothing becomes pronounced as  GNNs deepen and the node representations go indistinguishable, which significantly degrades the performance in the downstream tasks. To quantify the distance of node pairs, we employ the metric of \textit{Dirichlet Energy} to observe how SSM can alleviate the over-smoothing issue compared with the self-attention mechanism. Given the state $\boldsymbol{M}^{(t)}$ of nodes, the Dirichlet Energy $\mathbb{DE}(\boldsymbol{M}^{(t)})$ is defined as follows:
\begin{small}
        \begin{align}
            \mathbb{DE}(\boldsymbol{M}^{(t)})&=\mathrm{tr}\left( \boldsymbol{M}^{(t)} ({\bm I}-{{\bm A}^{(t)}}^{\top}{{\bm D}^{(t)}}^{-1}) {\boldsymbol{M}^{(t)}}^{\top}\right) \nonumber \\
            &= \frac12\sum {\bm A}^{(t)}_{ij}||\frac{\boldsymbol{M}^{(t)}(i)}{\sqrt{1+|\mathcal{N}_{out}^{(t)}(i)|}}-\frac{\boldsymbol{M}^{(t)}(j)}{\sqrt{1+|\mathcal{N}_{out}^{(t)}(j)|}}||_2^2,
        \end{align}
\end{small}
where $\mathrm{tr}(\cdot)$ denotes the trace of the matrix and ${\bm A}^{(t)}_{ij}$ is the $(i,j)$-th element of ${\bm A}^{(t)}$. Dirichlet Energy reflects the smoothness of adjacent node representations. Therefore, a relatively larger value of $ \mathbb{DE}(\boldsymbol{M}^{(t)})$ reveals that the representations of different nodes can be more distinct, making it easier to separate different nodes in classification tasks. Based on this insight, we investigate the potential over-smoothing issues of Coden and Coden-A by analyzing the Dirichlet Energy of their state matrices: 
\begin{lemma}\label{lemma:energy}
    Denoting the state matrix of Coden and Coden-A as $\boldsymbol{M}^{(t)}_C$ and $\boldsymbol{M}^{(t)}_A$ respective, we have:  $\mathbb{DE}(\boldsymbol{M}^{(t)}_A) \leq \mathbb{DE}(\boldsymbol{M}^{(t)}_C)$, if $F \cdot \sigma_{min}^2 (\prod_{i=s+1}^t \bar{\mathcal{A}}^{(i)} \bar{\mathcal{B}}^{(s)})\ge 1$ for $0 \leq s \leq t$, where $\sigma_{min}(\cdot)$ means the minimum eigenvalue of the matrix.
\end{lemma}
The detailed proof can be found in Appendix \ref{lemma:energy}. Note that we can easily apply the $L_2$ regularization during training to prevent the minimum eigenvalue of the parameter matrices from approaching zero \cite{ghorbani2019investigation,lewkowycz2020training}, which helps mitigate over-fitting and ensures better model stability. In the later experiments, the condition of  $F \cdot \sigma_{min}^2 (\prod_{i=s+1}^t \bar{\mathcal{A}}^{(i)} \bar{\mathcal{B}}^{(s)})\ge 1$ can generally satisfied (see Lemma \ref{lemma:spectral_floor} in the Appendix). The experimental comparison of these two models is provided in Sec. \ref{sec:ablation}.

The above in-depth analysis shows that while Coden produces an equivalent scheme with the self-attention mechanism, the non-trivial integration of SSM for temporal graph processing enhances both the efficiency and the quality of representations achieved by Coden. Moreover, we believe this direct association between SSM and the self-attention mechanism could shed light on the rationale behind the need for the application of SSM in temporal graphs.

\noindent
\subsection{Summary} 
As summarized in Table~\ref{table:method_comparison}, {\sc Coden} attains the most favorable complexity for continuous prediction among existing TGNN architectures. Together with the RNN--Attention duality established above, this indicates that {\sc Coden} combines the efficiency of lightweight recurrence with the expressive power of attention-style aggregation. To clarify how these properties arise, we relate {\sc Coden} to three major lines of temporal graph models as follows.

\noindent
$\bullet$ \textbf{Efficient update as single-snapshot methods.}
{\sc Coden} simplifies the traditional graph convolution (e.g., GCN~\cite{DBLP:conf/iclr/KipfW17}) and directly treats the parameter-free node embeddings as the base state.
This design avoids the usual $O(K m^{(t)} F)$ cost of running $K$-layer message passing on the updated graph, and thus enables fast incremental refresh of node states after each update.

\noindent
$\bullet$ \textbf{Optimized complexity within the RNN family.}
From the recurrent view, {\sc Coden} follows the same iterative pattern as RNN-based TGNNs, whose per-step update and training typically require $O(p n^{(t)} F^{2})$ time.
At the same time, {\sc Coden} employs a lazy-sampling strategy to control the effective sequence length, substantially improving the efficiency of iteration.
The trade-off is governed by the parameter $\lambda$: when $\lambda$ is small (e.g., $\lambda = O(1)$), {\sc Coden} processes each update individually and its runtime is comparable to that of standard RNN-based methods.

\noindent
$\bullet$ \textbf{Effectiveness comparable to attention-based methods.}
Attention-based TGNNs usually rely on explicit temporal attention, incurring $O\!\left(p (n^{(t)})^{2} F\right)$ complexity to achieve strong temporal memorization.
Our analysis in Section~\ref{sec:equ} shows that the unrolled dynamics of {\sc Coden} are equivalent to masked kernel attention, providing a theoretical explanation for its effectiveness.
This duality allows {\sc Coden} to retain the expressive benefits of attention-style aggregation, while enjoying substantially lower computational cost.

\vspace{-2mm} 
\section{Experiments}
In this section, we conduct comprehensive experiments to evaluate the key performances of {\sc Coden} against strong TGNN baselines.  In Sec. \ref{exp:acc_time} and \ref{sec:efficiency}, we separately provide the comparison based on the temporal node classification task, respectively highlighting the effectiveness and efficiency of our approach.  The ablation study and parameter analysis are provided in Sec. \ref{sec:ablation} and \ref{exp:hyper}. All experimental results are obtained with 10 runs on a Linux machine with an Intel(R) Xeon(R) Gold 6238R CPU @ 2.20GHz with 160GB RAM and an NVIDIA RTX A5000 with 24GB memory.
We provide an open-source implementation of our model at \url{https://anonymous.4open.science/r/Coden-46FF}.

\vspace{-2mm}
\subsection{Datasets}
We adopt five representative real-world dynamic datasets: \textit{DBLP} \cite{spikenet}, \textit{Tmall} \cite{lu2019temporal}, and three large-scale graphs, \textit{Reddit} \cite{graphsage}, \textit{Patent} \cite{hall2001} and \textit{Papers100M} \cite{hu2020open}. The statistics of the datasets are shown in Tab. \ref{table:dataset_statistics}. Within these datasets, the training, validation, and test sets are randomly allocated in proportions of 70\%, 10\%, and 20\% respectively. To simulate scenarios that necessitate continuous prediction, we adopt the experimental framework outlined in \cite{zheng2022instant} and \cite{IDOL}, where the graph is segmented into an initial graph and $|T|$ batches of edge sequences. Then, each batch of edges will be added at distinct time steps in a dynamically evolving state. All methods are evaluated through a pipeline encompassing both training and inference.

\setlength{\tabcolsep}{2mm}{
\small
\begin{table}[htb]
\centering
\small
\caption{\small Statistics of the datasets. $F$, $C$, and $|T|$ stand for the dimension of attributes, the number of classes, and the number of time steps to be predicted.}
 
\begin{tabular}{c|c|c|c|c|c} \toprule[1pt]
\textbf{Datasets} & Nodes ${n}$ & Edges ${m}$ & ${F}$ & ${C}$ & ${|T|}$\\ \midrule[0.5pt]
\textit{DBLP} & 28,085 & 236,894 & 128 & 10& 26 \\
\textit{Tmall} &  577,314	&  4,807,545 &  80&  5 & 19 \\
\textit{Reddit} & 227,853& 114,615,892 & 602& 40 & 16  \\
\textit{Patent} &2,738,012&  13,960,811 & 128& 6 & 16  \\
\textit{Papers100M} & 111,059,956 & 1,615,685,872 & 128 & 172 & 21 \\
\bottomrule[1pt]
\end{tabular}
\label{table:dataset_statistics}
\end{table}
}

\subsection{Baseline Methods}
We compare the proposed method {\sc Coden} with several state-of-the-art TGNN methods. They include: 
(i) the single-snapshot method Instant \cite{zheng2022instant}; (ii) the RNN-based methods 
TGCN \cite{tgcn},  EvolveGCN \cite{evolvegcn}, MPNN \cite{MPNN}, TGL+TGN \cite{tgl}\footnote{We run TGL with 4 GPUs in parallel and adopt the most efficient backbone TGN \cite{tgn}.} and CAWN \cite{CAWN}; (iii) the attention-based methods DNNTSP \cite{yu2020predicting}, DyGFormer \cite{dyformer}, and ASTGCN \cite{ASTGCN}.
Moreover, we also include the results of the recent competitive models SpikeNet \cite{spikenet} and Zebra \cite{li2023zebra}. 
To this end, we include several strong baselines focusing on the single snapshot, such as classical GCN \cite{DBLP:conf/iclr/KipfW17}, GraphSage \cite{graphsage}, and the scalable methods SCARA \cite{scara}.
Since most of the compared RNN-based and attention-based methods are designed for small-scale graphs, we adopt the neighboring sampling techniques \cite{graphsage} for these methods when applied on the \textit{Reddit}, \textit{Patent} and \textit{Papers100M} datasets to avoid the out-of-memory problem. We summarize the experimental settings for all baselines in Tab. \ref{table:exp setting}, where the common parameters, such as learning rate and hidden size, are maintained consistently for Coden. Specifically, we set the threshold $\epsilon = 1e^{-7}$, $\lambda = 0.1$,  $F' =16$, and $\alpha = 0.2$ by default in Coden.

{\small
\begin{table}[t]
\centering
\caption{\small Parameter settings. Here "lr" means the learning rate, "$K$" means the number of convolution layers, "hidden" means the hidden size of the network, and "batch number" means the number of neighbor sampling for baselines.}
 
\begin{tabular}{c|c|c|c|c} \toprule[1pt]
\textbf{Datasets} & lr &K&  hidden & batch number\\ \midrule[0.5pt]
\textit{DBLP} & 1e-3 &4& 1024 & 12 \\
\textit{Tmall} & 1e-3 &  4&1024&  12   \\
\textit{Reddit} &1e-3& 4&512& 12   \\
\textit{Patent} &1e-3& 4&512& 12  \\
\textit{Papers100M} &1e-3& 2&512& 12  \\
\bottomrule[1pt]
\end{tabular}
\label{table:exp setting}
\end{table}
}

\vspace{-2mm}
\subsection{Accuracy Comparison}\label{exp:acc_time}
To simulate the evolving scenarios where the edge updates happen periodically, we first remove all edges from the graph and then progressively re-add them in chronological order according to the time steps. With $|T|$ time steps in our setting, each dataset can yield $|T|$ distinct results. Note that this setting is substantially different from a general experimental setting in TGNNs, which typically focuses only on the fully formed final graph and fails to reflect the evolving process of the graph. Collectively, our objective is to evaluate the performance of all compared methods in a scenario requiring continuous prediction.

\setlength{\tabcolsep}{1.2mm}{
\begin{table*}[t]
\centering
\footnotesize
 
\caption{\small The average, the best and the worst accuracy (\%, denoted as Ave. and Bes. respectively) across all time steps. "OOM" stands for out of memory on a GPU with 24GB memory. The best results are ranked  \rkat{first}, \rkbt{second} for TGNN methods.}
 
\label{table:alldata}
\begin{tabular}[t]{c|cc|cc|cc|cc|cc} \toprule[1pt]
\multirow{2}{*}{\bf Method} & \multicolumn{2}{c|}{\bf DBLP}&\multicolumn{2}{c|}{\bf Tmall} &\multicolumn{2}{c|}{\bf Reddit} & \multicolumn{2}{c|}{\bf Patent} & \multicolumn{2}{c}{\bf Papers100M}\\ & {\bf Ave.} & {\bf  Bes.} 
 & {\bf Ave.} & {\bf  Bes.}  & {\bf Ave.} & {\bf  Bes.}  & {\bf Ave.} & {\bf  Bes.} & {\bf Ave.} & {\bf  Bes.}  \\ \midrule[0.5pt]
 GCN &$72.32_{\pm 0.12}$ &$74.24_{\pm 0.22}$  &$60.11_{\pm 0.15}$ &$62.36_{\pm 0.46}$&$90.92_{\pm 0.45}$ &{$93.30_{\pm 0.84}$ }&$81.26_{\pm 0.54}$ &$83.95_{\pm 0.60}$&OOM&OOM\\
 GraphSage &$72.53_{\pm 0.64}$ &$74.55_{\pm 0.61}$ & $60.25_{\pm 0.40}$ & {$64.62_{\pm 0.71}$}  &$90.65_{\pm 0.26}$& {$92.95_{\pm 0.37}$}& {$82.95_{\pm 0.75}$} &{$83.21_{\pm 0.44}$}&OOM&OOM\\ 
 SCARA & {$73.57_{\pm 0.29}$} &  {$75.68_{\pm 0.66}$}  & {$61.32_{\pm 0.43}$} &$63.25_{\pm 0.11}$  & {$91.65_{\pm 0.24}$}&$92.29_{\pm 0.50}$ & {$82.44_{\pm 0.32}$ }&{$83.55_{\pm 0.29}$} 
 &$62.22_{\pm 0.46}$&$63.40_{\pm 0.21}$\\
 Instant &{$74.25_{\pm 0.21}$} & {$75.22_{\pm 0.44}$}  &  {$61.22_{\pm 0.87}$}&{$64.26_{\pm 0.43}$} &{$92.10_{\pm 0.67}$} & {$92.34_{\pm 0.29} $}&{$81.95_{\pm 0.37}$}&$82.99_{\pm 0.59}$&$62.75_{\pm 0.17}$&$64.15_{\pm 0.35}$\\\midrule[0.5pt]
TGCN &$74.17_{\pm 0.83}$ &$75.65_{\pm 0.75}$  &$60.00_{\pm 0.93}$ &$61.92_{\pm 0.71}$ & $92.13_{\pm 0.83}$&$93.73_{\pm 0.44}$  & $82.51_{\pm 0.68}$&$83.39_{\pm 0.60}$&OOM&OOM\\
 ASTGCN &{$75.04_{\pm 0.54}$}&$76.7_{\pm 0.19}$  &$61.80_{\pm 0.72}$ & \rkb{$66.59_{\pm 0.62}$} &$91.66_{\pm 0.69}$& $93.79_{\pm 0.81}$ & {$82.52_{\pm 0.70}$}&{$83.36_{\pm 0.56}$}&OOM&OOM\\
 EvolveGCN & $73.27_{\pm 0.48}$&$76.09_{\pm 0.54}$ &$64.27_{\pm 0.63}$ & $66.05_{\pm 0.39}$& $91.33_{\pm 0.53}$& {$93.50_{\pm 0.27}$}&$82.74_{\pm 0.46}$&$83.51_{\pm 0.22}$&OOM&OOM\\
 MPNN & $74.58_{\pm 0.17}$&\rkb{$77.02_{\pm 0.19}$} &$61.81_{\pm 0.36}$& $65.04_{\pm 0.60}$ &$91.09_{\pm 0.30}$ &$93.08_{\pm 0.64}$ & \rkb{$83.25_{\pm 0.52}$}& \rkb{$84.11_{\pm 0.49}$}&OOM&OOM\\ 
  CAWN & $73.56_{\pm 0.27}$&$75.58_{\pm 0.67}$ &$62.34_{\pm 0.39}$ & $65.35_{\pm 0.21}$& $91.54_{\pm 0.76}$& {$93.53_{\pm 0.65}$}&$82.21_{\pm 0.36}$&$83.62_{\pm 0.72}$&OOM&OOM\\
 DNNTSP & $74.25_{\pm 0.32}$&$75.87_{\pm 0.76}$ &\rkb{$64.29_{\pm 0.88}$}& $65.9_{\pm 0.83}$ &OOM&OOM&OOM&OOM&OOM&OOM\\
 SpikeNet & $74.83_{\pm 0.58}$&$75.72_{\pm 0.71}$ &$64.21_{\pm 0.40}$& $66.02_{\pm 0.66}$&\rkb{$92.19_{\pm 0.41}$ }&\rkb{$93.83_{\pm 0.37}$} &$81.86_{\pm 0.63}$ &$82.52_{\pm 0.47}$&OOM&OOM\\
 Zebra & \rkb{$75.04_{\pm 0.40}$}&$75.97_{\pm 0.33}$ &$63.14_{\pm 0.27}$& $65.12_{\pm 0.89}$&{$91.77_{\pm 0.42}$ }&{$92.94_{\pm 0.12}$} &$82.86_{\pm 0.66}$ &{$83.98_{\pm 0.27}$}&OOM&OOM\\
 TGL+TGN & {$73.12_{\pm 0.34}$}&$76.91_{\pm 0.60}$ &$62.56_{\pm 0.77}$& $65.02_{\pm 0.89}$&{$92.11_{\pm 0.55}$ }&{$93.53_{\pm 0.33}$} &$81.20_{\pm 0.31}$ &{$83.48_{\pm 0.66}$}&\rkb{$62.88_{\pm 0.26}$}&\rkb{$64.33_{\pm 0.53}$}\\
 DyGFormer & {$73.88_{\pm 0.54}$}&$76.79_{\pm 0.29}$ &$63.04_{\pm 0.30}$& $65.94_{\pm 0.33}$&{$91.54_{\pm 0.61}$ }&{$93.32_{\pm 0.62}$} &$82.16_{\pm 0.47}$ &{$83.59_{\pm 0.25}$}&OOM&OOM\\
 {{\sc Coden}} &\rka{$ 76.35_{\pm 0.17}$} & \rka{$77.23_{\pm 0.27}$}&\rka{$65.14_{\pm 0.19}$}& \rka{$66.7_{\pm 0.23}$} &\rka{$92.61_{\pm 0.19}$}&\rka{$94.17_{\pm 0.30}$} &\rka{$83.74_{\pm 0.39}$}&\rka{$84.53_{\pm 0.14}$} & \rka{$64.89_{\pm 0.16}$} & \rka{$66.45_{\pm 0.33}$}\\\bottomrule[1pt]
\end{tabular} 
 
\hfill
\centering
\end{table*}
}
\noindent
\textbf{Overall statistic.} To synthetically evaluate the performance of all compared methods, we first provide detailed statistics including the average, best, and worst accuracy across all prediction times, as shown in Tab. \ref{table:alldata}. We separate the methods according to fact whether they incorporate the temporal information. Since most of methods run out of memory on the \textit{Papers100M} dataset, we mainly focus on the results from the other four datasets in Tab. \ref{table:alldata}.

\noindent
\textbf{Temporal information is crucial for accurate prediction.} 
Methods designed for single snapshots, such as GCN, GraphSage, SCARA, and Instant, fall short compared to other TGNN approaches. This disparity arises because these methods focus exclusively on individual snapshots, overlooking the evolving dynamics within the graph. This underscores the critical role of temporal information in achieving more accurate predictions. Therefore, our comparison will primarily focus on methods that process temporal information in the following sections.

\noindent
\textbf{{\sc \textbf{Coden}} exhibits desirable performance across various time steps.}
We then compare the accuracy performance of {\sc Coden} with other strong competitors. As shown in Fig. \ref{fig:accuracy}, we present the Micro-F1 scores on \textit{Reddit} and \textit{Patent} across different prediction times to demonstrate the accuracy comparison.
As an overview, one key observation is indicated by these comparison results: \textit{{\sc Coden} can match or even outperform the state-of-the-art methods in terms of prediction accuracy in most cases.} Furthermore, while other methods exhibit significant fluctuations across different time steps, {\sc Coden} maintains more robust performance, with accuracy improving steadily. We attribute this noteworthy performance to two primary factors: (i) Our incorporated \textit{Embedding Update} algorithm ensures the accuracy of embeddings in an evolving process, thereby enhancing their quality and improving robustness during updates. (ii) The embedding sequences generated through the lazy-sampling strategy enable effective information compression when integrated with the SSM module. This distinctive information compression mechanism allows {\sc Coden} to selectively retain valuable information throughout the evolving process, thereby effectively filtering useful data.



\subsection{Efficiency Comparison}\label{sec:efficiency}

\setlength{\tabcolsep}{0.6mm}{
\begin{table}[htb]
\centering
{
\caption{The total training time ($s$, denoted as Train.) and inference time ($s$, denoted as Infer.).}
 
\label{table:average_time}
\tiny
\begin{tabular}[t]{c|c|c|c|c|c} \toprule[1pt]
\multirow{1}{*}{\bf Method} & \multicolumn{1}{c|}{\bf DBLP}&\multicolumn{1}{c|}{\bf Tmall} &\multicolumn{1}{c|}{\bf Reddit} & \multicolumn{1}{c|}{\bf Patent} & \multicolumn{1}{c}{\bf Papers100M}\\ \midrule[0.5pt]
TGCN &$11.12$ ($0.04$) & $214.23$( $0.08$) &\rkb{$2.29K$  ($12.16$)}& \rkb{$4.62K$ ($19.54$)} & OOM \\
 ASTGCN &$439.56$ ($0.14$) &$3.09K$ ($1.50$) & $51.03K$ ($13.54$)& $77.59K$ ($34.58$) & OOM\\
 EvolveGCN &\rkb{$8.43$ ($0.06$)} & \rkb{$123.54$($0.08$)} & $2.61K$ ($12.56$) & $5.16K$ ($19.45$)& OOM\\
 MPNN &$21.36$ ($0.04$) & $303.65$ ($0.12$)& $2.80K$ ($13.10$) & $6.43K$ ($19.98$)& OOM\\ 
 DNNTSP &$3.86K$ ($0.10$) & $17.35K$ ($5.04$)& OOM&OOM& OOM\\
 SpikeNet & $255.11$ ($0.12$) & $546.21$ ($0.55$)& $3.91K$ ($15.26$) & $12.83K$ ($29.56$)& OOM\\
 Zebra &$435.26$ ($0.14$) & $946.28$ ($0.76$)& $4.03K$ ($15.28$) & $20.34K$ ($32.71$)& OOM\\
 TGL &$868.65$ ($0.15$) & $1.96K$ ($1.22$)& $18.48K$ ($40.99$) & $38.56K$ ($68.98$)& $354K$ ($150.45$)\\
 {{\sc Coden}} & \rka{$6.00$ ($0.02$)}& \rka{$23.68$($0.05$)} & \rka{$1.31K$ ($8.15$)}&  \rka{$1.73K$ ($11.69$)} 
 &$12.30K$ ($20.15$)
 \\\bottomrule[1pt]
\end{tabular} 
\hfill
\centering
}
\end{table}
}

\noindent
 {\textbf{{\sc \textbf{Coden}} achieves a speedup of up to $\mathbf{44.80\times}$ in time consumption and finishes continuous predictions within 3.3 hours averagely on the billion-scale graph \textit{Papers100M}.}} To ensure a fair comparison of efficiency, we evaluate the total time consumption, encompassing \textit{ the network training (100 epochs), and inference}, which collectively constitute a single complete prediction. We report the overall training time and the inference time for each method, as shown in Tab. \ref{table:average_time}. The experimental results highlight the superiority of our model in achieving scalability through the learning phases. On the large-scale graph such as \textit{Patent}, {\sc Coden} achieves $2.67-44.80\times$ acceleration in training time, along with improved inference speed. 
 
 To demonstrate this difference clearly, 
we record the detailed results on \textit{Reddit} and \textit{Patent} datasets across all time steps in Fig. \ref{fig:time_patent}. 
It is implied that all baselines experience an increase in training time as the number of edge updates grows, escalating rapidly with the data size. 
In contrast, {\sc Coden} exhibits only a slight increase in the time consumption (e.g., ranging from $1.71K$s to $1.86K$s on \textit{Patent}), significantly improving training efficiency and reducing system workload. 
 {Furthermore, both {\sc Coden} and TGL are capable of performing continuous prediction on the \textit{Papers100M} dataset, but {\sc Coden} significantly outperforms TGL in terms of efficiency, efficiently handling the scale where other TGNN methods fail in memory limitations.}
We believe this superior performance derives from two reasons: (i) Different from other baselines requiring graph propagation from scratch with updates, 
the embedding updating process of {\sc Coden} is conducted with an incremental form, avoiding redundant computation and reducing the time consumption of training at each prediction time. (ii) {\sc Coden} leverages the linear recurrence mechanism of SSM to retain long-range historical information in node states. These two properties create a more scalable pipeline and enhance computational speed.


\begin{figure}[htb]
\centering
 
\subfigure[\textit{DBLP}]{
\includegraphics[width=1.55in]{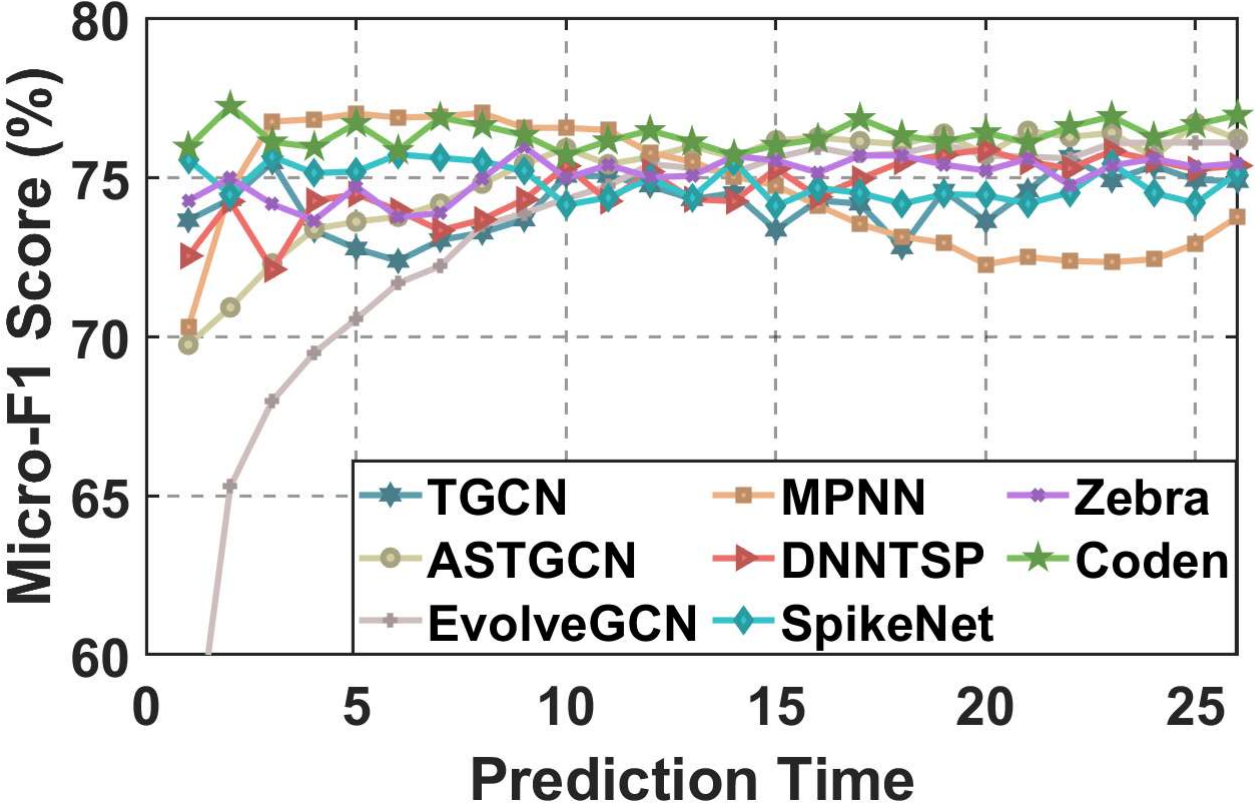}}
\subfigure[\textit{Tmall}]{
\includegraphics[width=1.55in]{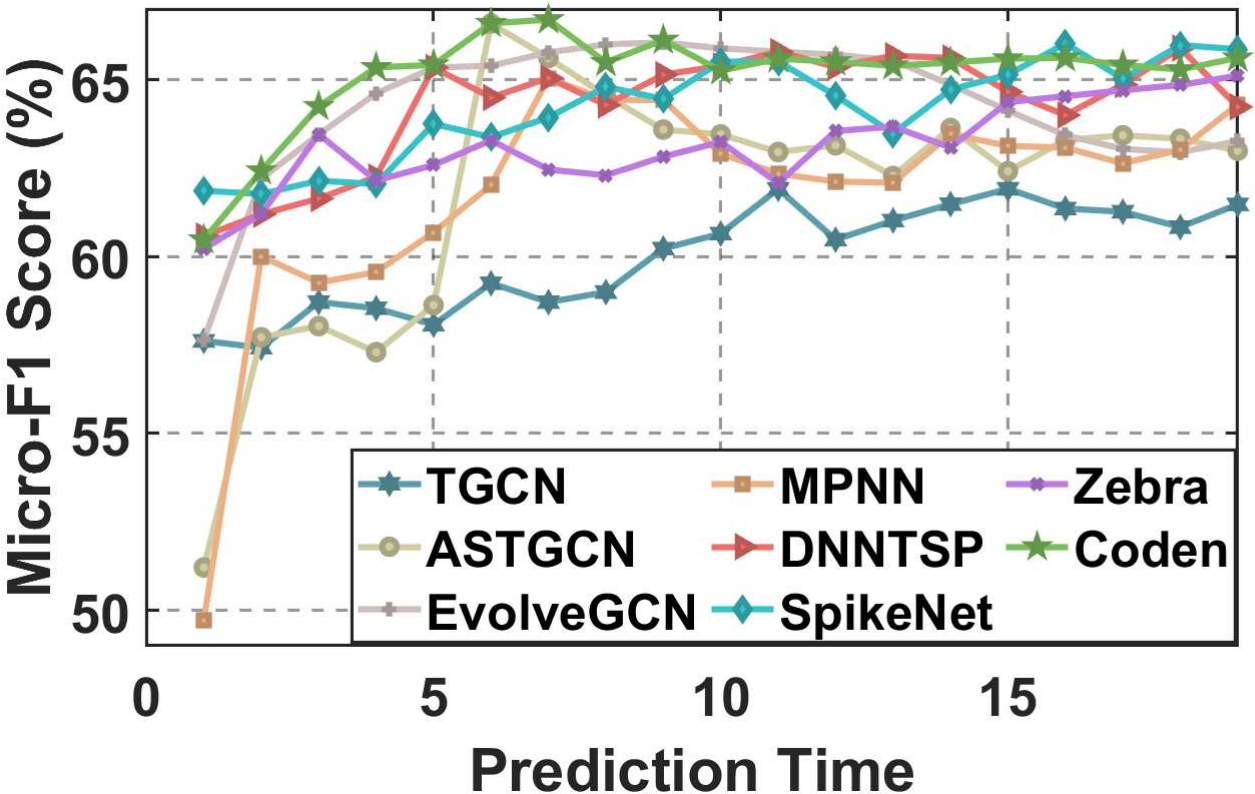}
}

\subfigure[\textit{Reddit}]{
\includegraphics[width=1.55in]{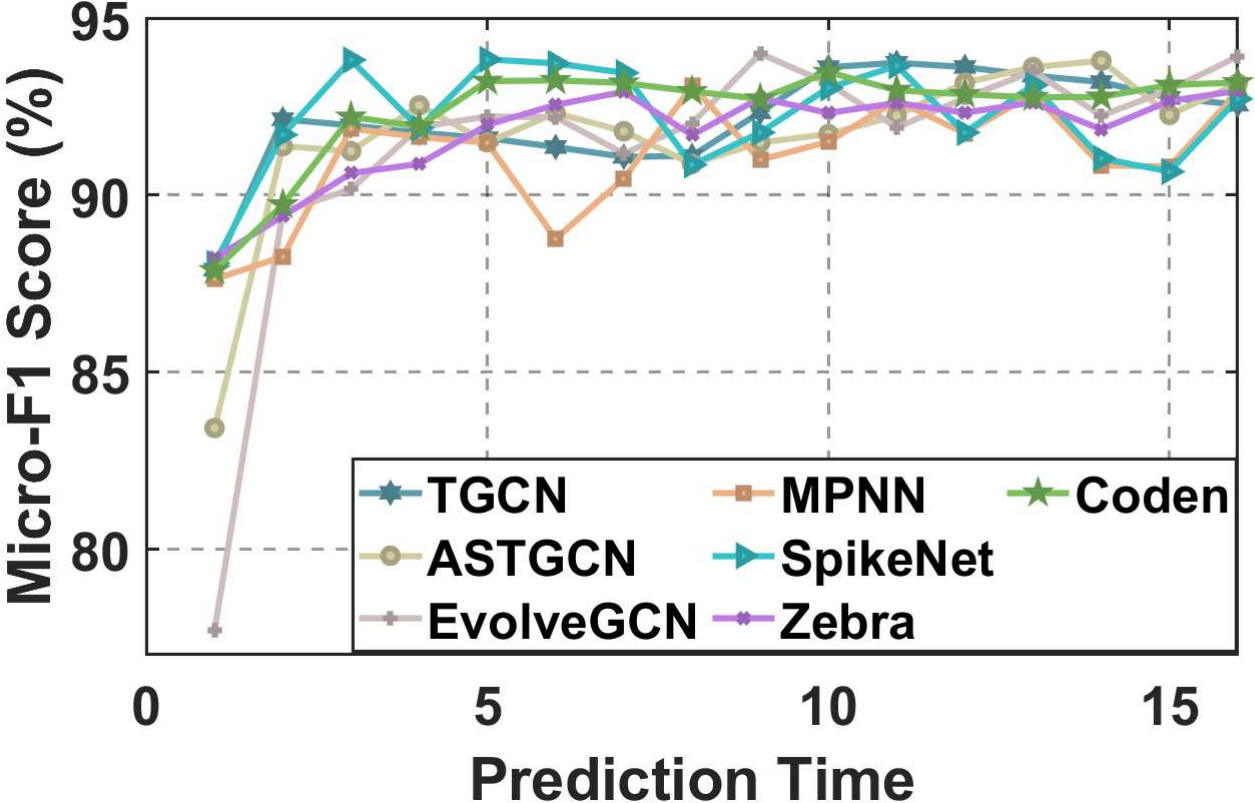}}
\subfigure[\textit{Patent}]{
\includegraphics[width=1.55in]{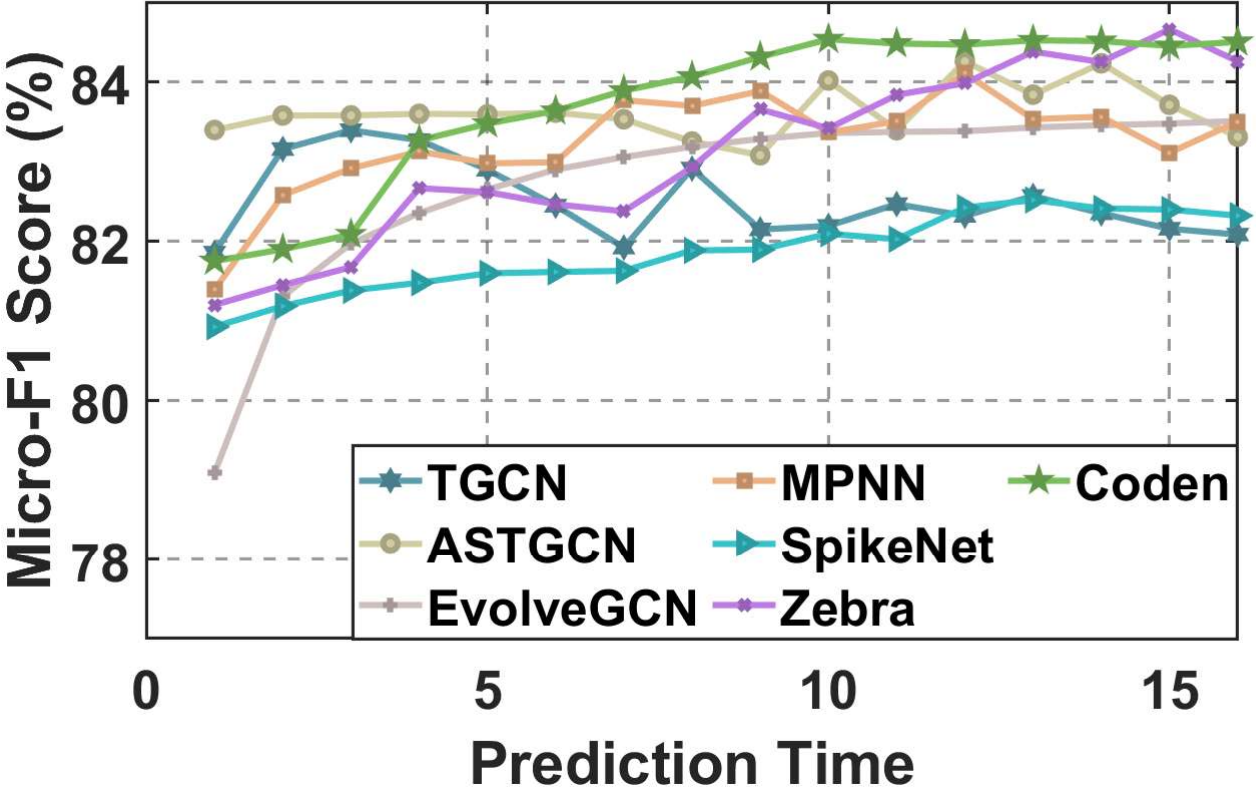}
}
 
\caption{Micro-F1 scores for each prediction time of four datasets.}

\label{fig:accuracy}
\end{figure}

\begin{figure}[tb]
\centering
 
\subfigure[\textit{DBLP}]{
\includegraphics[width=1.55in]{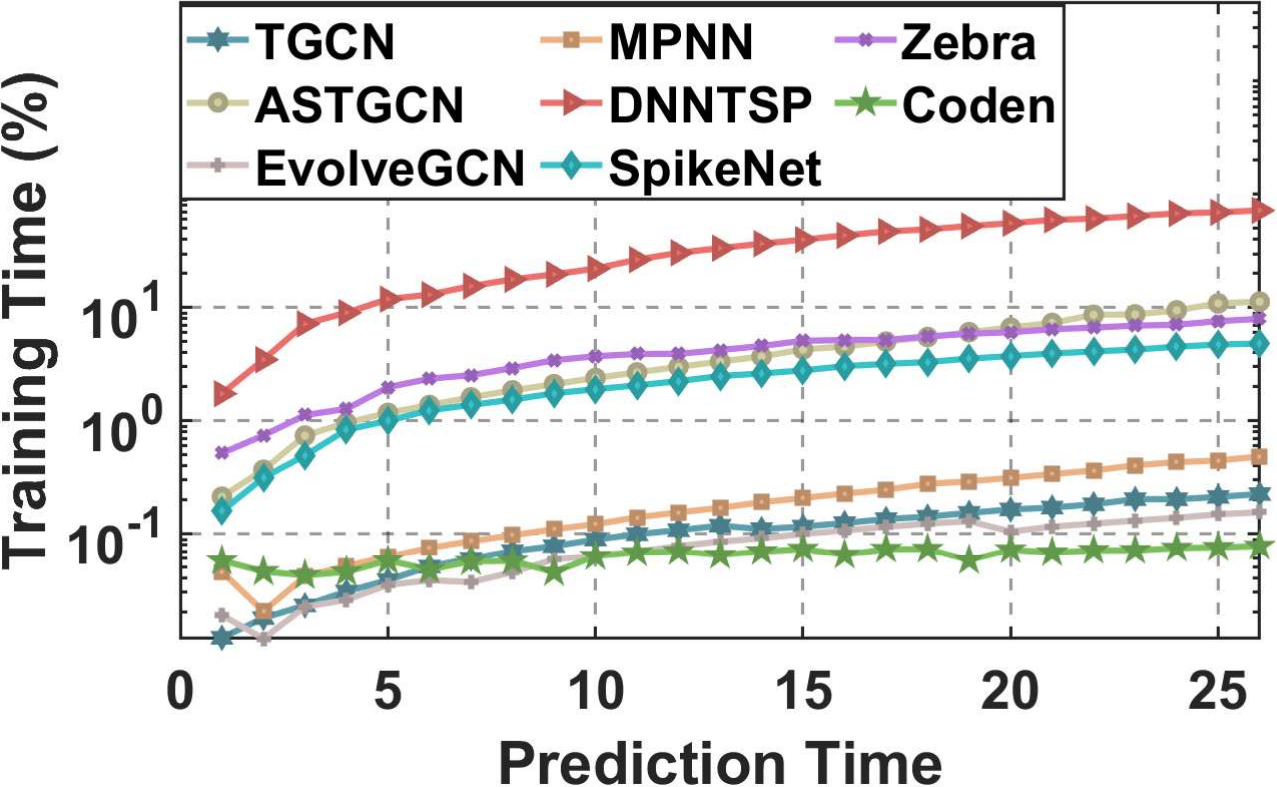}}
\subfigure[\textit{Tmall}]{
\includegraphics[width=1.55in]{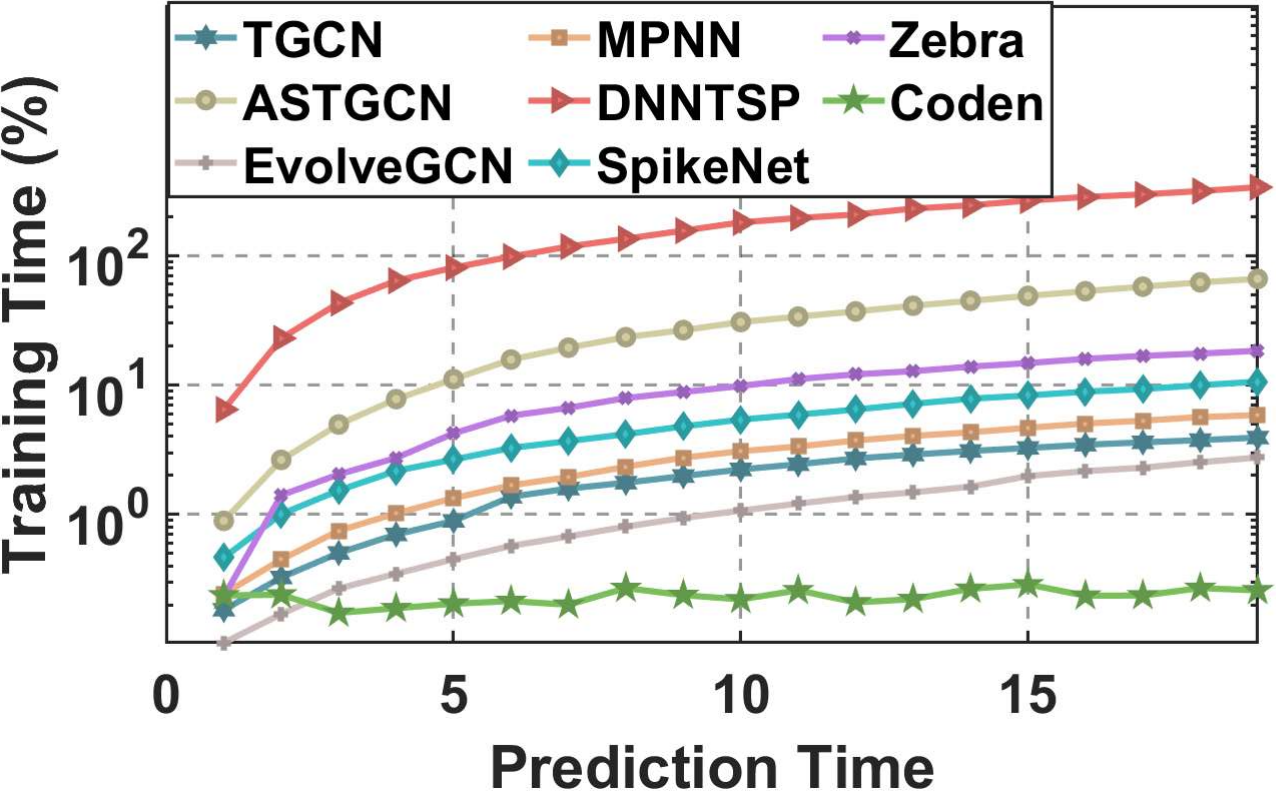}
}

\subfigure[\textit{Reddit}]{
\includegraphics[width=1.55in]{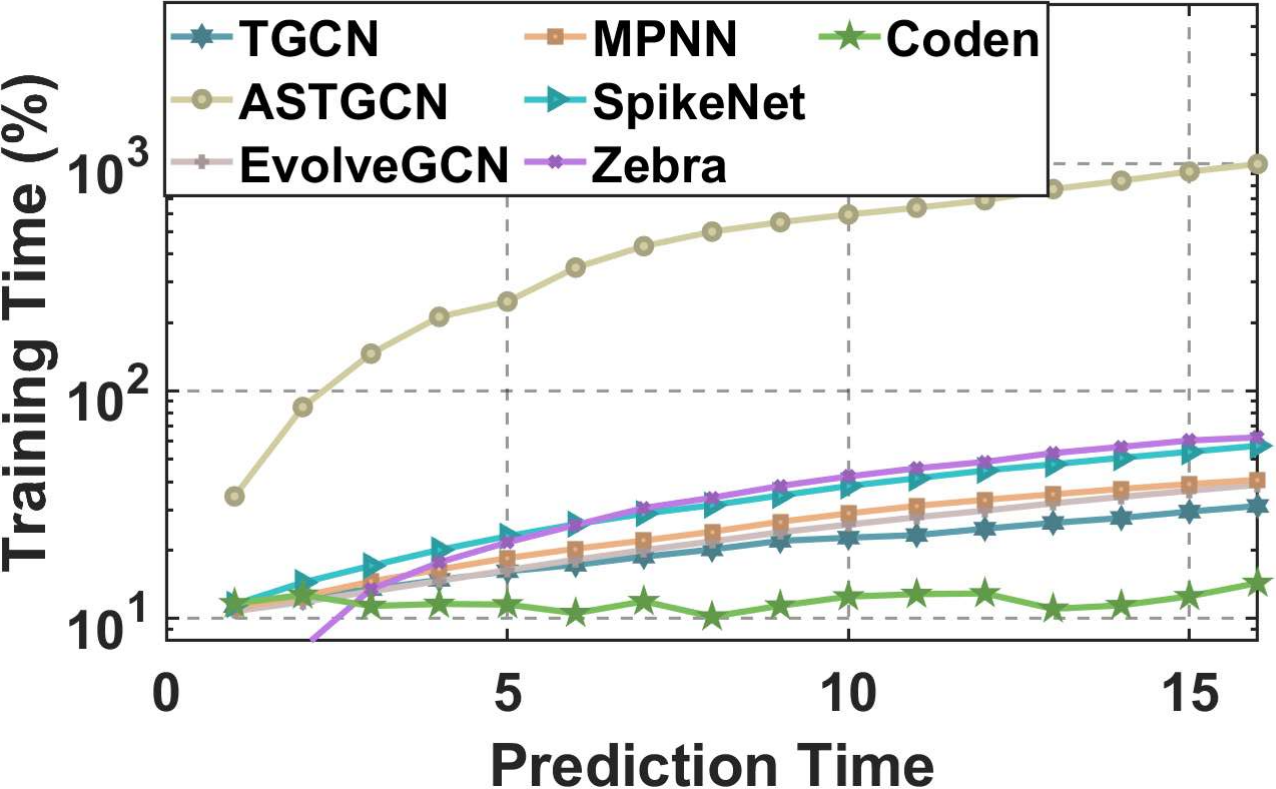}}
\subfigure[\textit{Patent}]{
\includegraphics[width=1.55in]{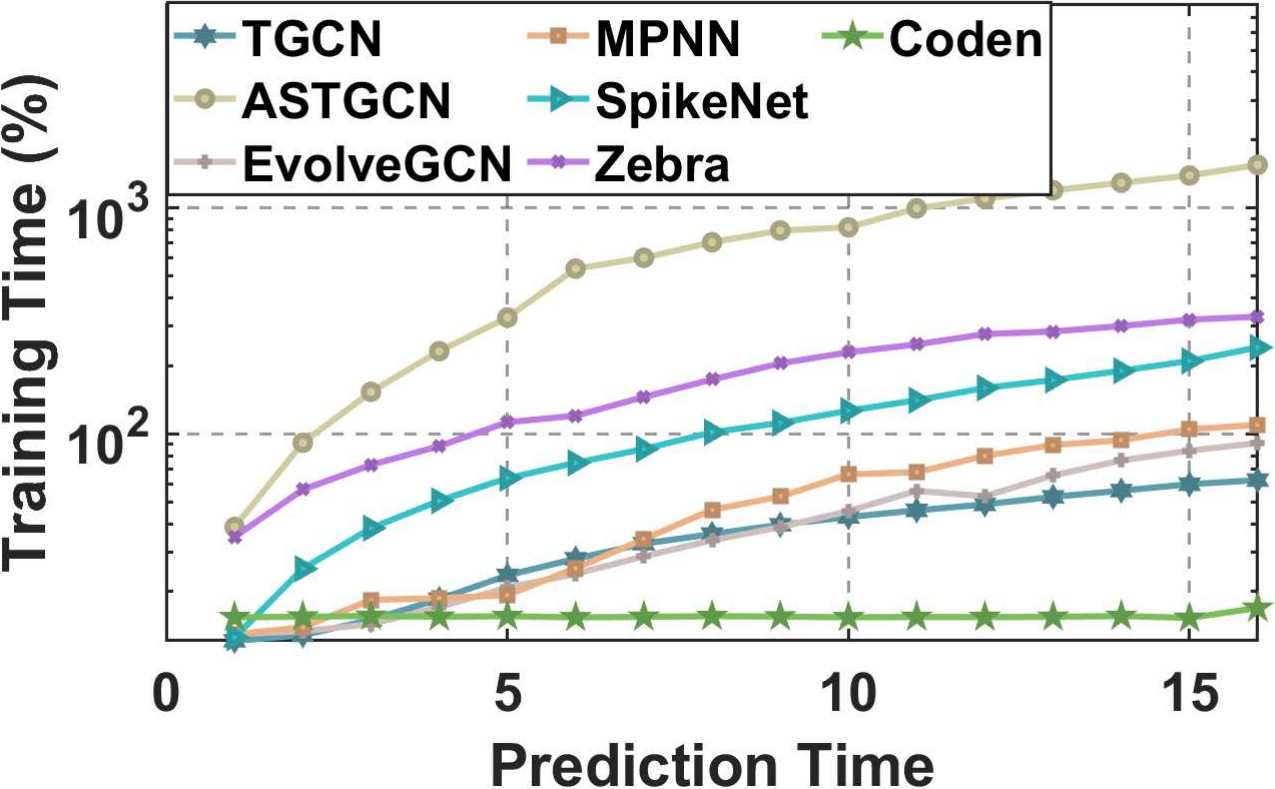}
}
 
\caption{Average training time (per epoch) for each prediction time.}
\label{fig:time_patent}
\end{figure}

\subsection{Ablation Study}\label{sec:ablation}
We then perform an ablation study to demonstrate the contribution of each module in {\sc Coden}.
We compare our model with its three variants: {{\sc Coden}-S} modifies {\sc Coden} by removing the temporal structure and employing an MLP classifier on the current node embeddings for straightforward prediction; {\sc Coden}-R inherits the overall architecture of {\sc Coden} while imposing simple constraints on the discretized SSM parameters (Lemma \ref{lem:coden_gru_form}); {\sc Coden}-A is the alternative model utilizing a attention structure for message passing in Sec. \ref{sec:equ}. As shown in Tab. \ref{table:ablation}, {\sc Coden}-S is the most efficient model as it lacks dependency between the current and previous graphs. The accuracy gap compared to {\sc Coden} highlights that incorporating temporal information can substantially enhance prediction accuracy. On the contrary, {\sc Coden}-A effectively incorporates historical information derived from the attention structure, producing more precise representations compared to {\sc Coden}-S. However, due to its high complexity, the accuracy improvement of {\sc Coden}-A does not offset the loss in efficiency. Moreover, we observed that despite {\sc Coden}-A's significantly longer processing time for temporal information, its predictive performance still falls short of {\sc Coden}. This shortfall may stem from an over-smoothing effect, as elucidated in Sec. \ref{sec:super}.  {Compared to {\sc Coden}-R, {\sc Coden} achieves consistently better predictive performance, showing that relaxing the RNN-style constraints yields a stronger state update mechanism. }

\subsection{Effect of Hyperparameters}\label{exp:hyper}
We also examine the sensitivity of {\sc Coden} to hyperparameters. First, we examine a critical threshold in our framework: the sampling threshold $\lambda$, which determines the embedding sequences appearing in the SSM module and further affects the training time. Second, acting like the RNN gating mechanism, the hidden dimension $F'$ can influence the degree of information compression \cite{mamba}. Hence, we vary it to investigate its impact on our model.
The following experimental results are reported based on the \textit{DBLP} and \textit{Patent} datasets.

\setlength{\tabcolsep}{1mm}{
\begin{table}[t]
\centering
\footnotesize
 
\caption{\small The average accuracy and the time consumption across all time steps for each variant of {\sc Coden}.}
 
\label{table:ablation}
\begin{tabular}[t]{c|cc|cc|cc|cc} \toprule[1pt]
\multirow{2}{*}{\bf Method} & \multicolumn{2}{c|}{\bf DBLP}&\multicolumn{2}{c|}{\bf Tmall} &\multicolumn{2}{c|}{\bf Reddit} & \multicolumn{2}{c}{\bf Patent}\\ & {\bf Ave.} &  {\bf Tim.}
 & {\bf Ave.} & {\bf Tim.} & {\bf Ave.} & {\bf Tim.} & {\bf Ave.}  & {\bf Tim.}\\ \midrule[0.5pt]
 {\sc Coden}-S &{$74.88$}  &\rka{$0.04$} & {$61.24$}&\rka{$0.18$}&{$91.11$} &\rka{$6.86$}&$81.78$ & \rka{$11.99$} \\
 {\sc Coden}-R & $74.33$&\rkb{$0.08$}&$62.32$& $0.26$&{$91.21$ }&$11.95$ & $81.79$& $16.36$\\
 {\sc Coden}-A & \rkb{$76.07$}&$3.19$ &\rkb{$64.69$}&$27.69$& \rkb{$91.67$}&$412.38$&\rkb{$83.09$}&$672.64$ \\
 {{\sc Coden}} &\rka{$ 76.35$} & \rkb{$0.08$}&\rka{$65.14$}&  \rkb{$0.28$}&\rka{$92.38$}&\rkb{$12.10$}&\rka{$83.74$}&\rkb{$17.26$} \\\bottomrule[1pt]
\end{tabular} 
\hfill
\centering

\end{table}
}

\noindent
\textbf{Threshold $\lambda$.}
To achieve the sequence discretization, we employ a threshold $\lambda$ to delay the embedding sampling until the estimated error surpasses this upper bound. This hyperparameter determines the number of sampled embeddings and consequently affects the model's training time. 
As shown in Fig. \ref{fig:lambda and F}, we present the average training time and the average accuracy across all time steps by varying the threshold $\lambda$ within a reasonable range. It has been noted that there is a trade-off between training efficiency and accuracy. Specifically, increasing the value of $\lambda$ can significantly reduce the training time at the expense of prediction accuracy. When $\lambda$ surpasses a certain value ($0.05$ for \textit{DBLP} and $0.1$ for \textit{Patent}) and becomes even larger, the model samples only a minimal number of embeddings, which is inclined to a scenario akin of processing a single snapshot. This suggests again that incorporating dynamic modeling is crucial for improving prediction accuracy.

\vspace{1mm}
\noindent
\textbf{Hidden dimension $F'$.} The hidden size $F'$ in SSM acts as the shape of selected information from the input. The experimental result from Fig. \ref{fig:lambda and F} shows that various settings typically do not significantly influence the prediction accuracy of Coden, and increasing the value of $F'$ from 16 yields only a marginal improvement of less than $0.5\%$. However, an overlarge size of the projected information ($F'>64$) will typically degrade the efficiency of training. Consequently, we set $F' = 16$ as our default configuration to deliver the optimal performance.

\begin{figure}[htb] 
\centering
\subfigure[DBLP]{
\includegraphics[width=1.54in]{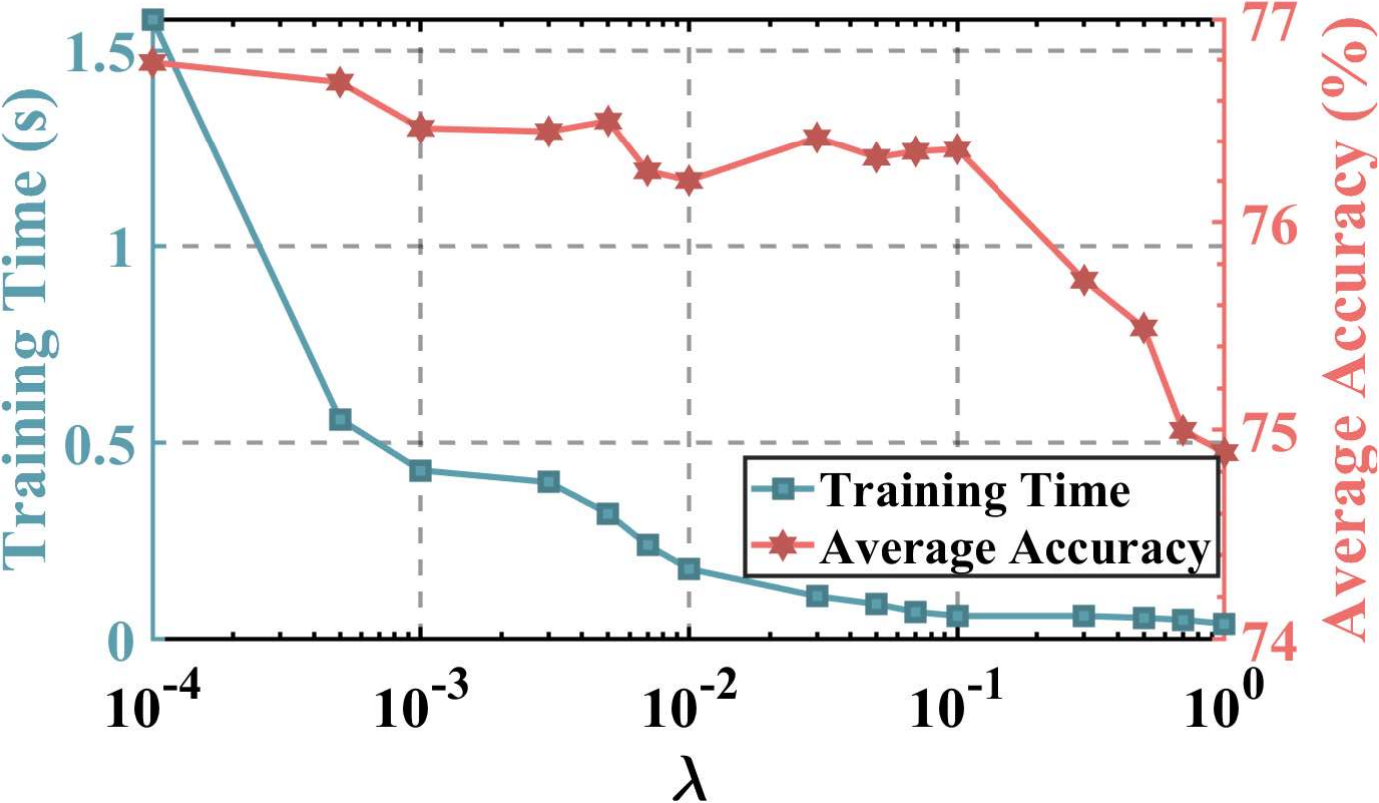}\label{fig:lambda_d}
\includegraphics[width=1.56in]{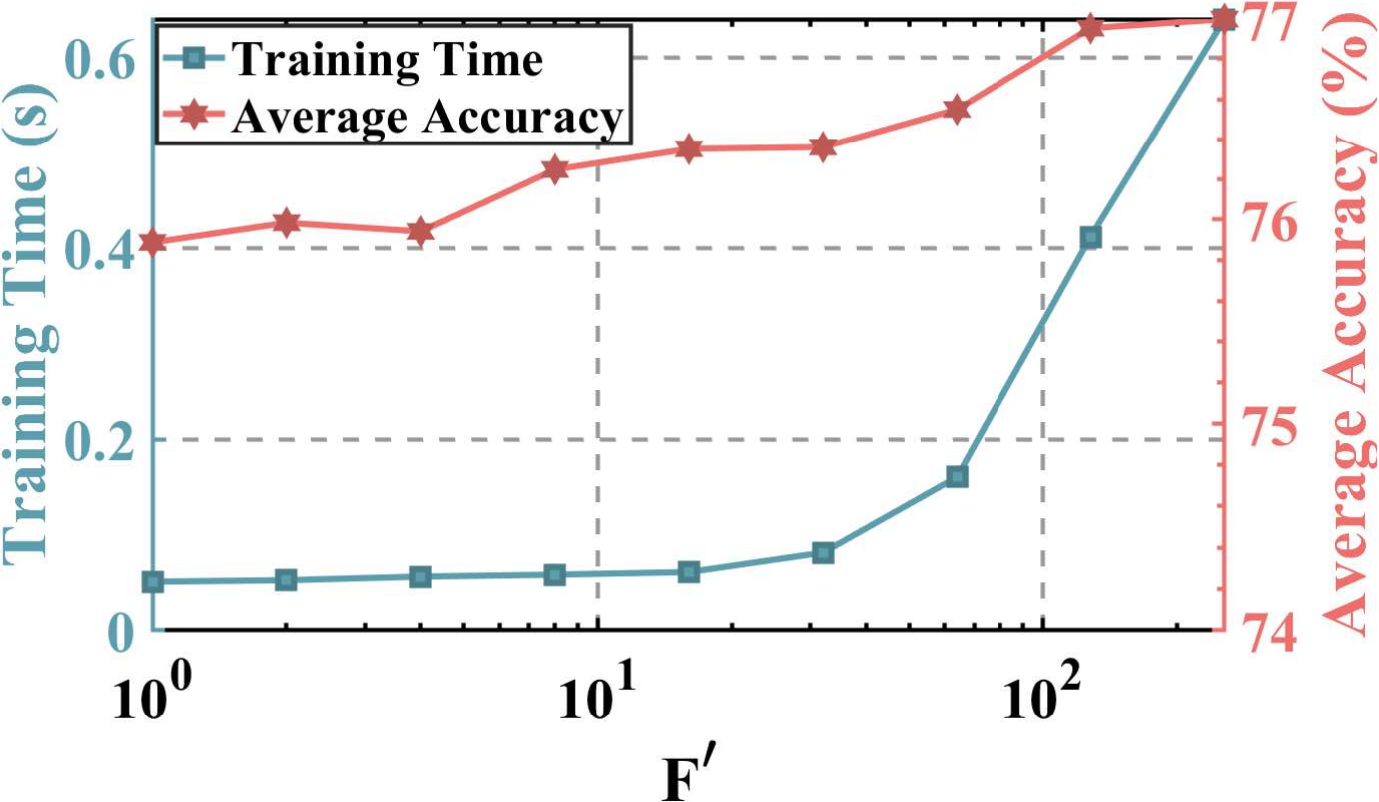}\label{fig:F_d}
}

\subfigure[Patent]{
\includegraphics[width=1.55in]{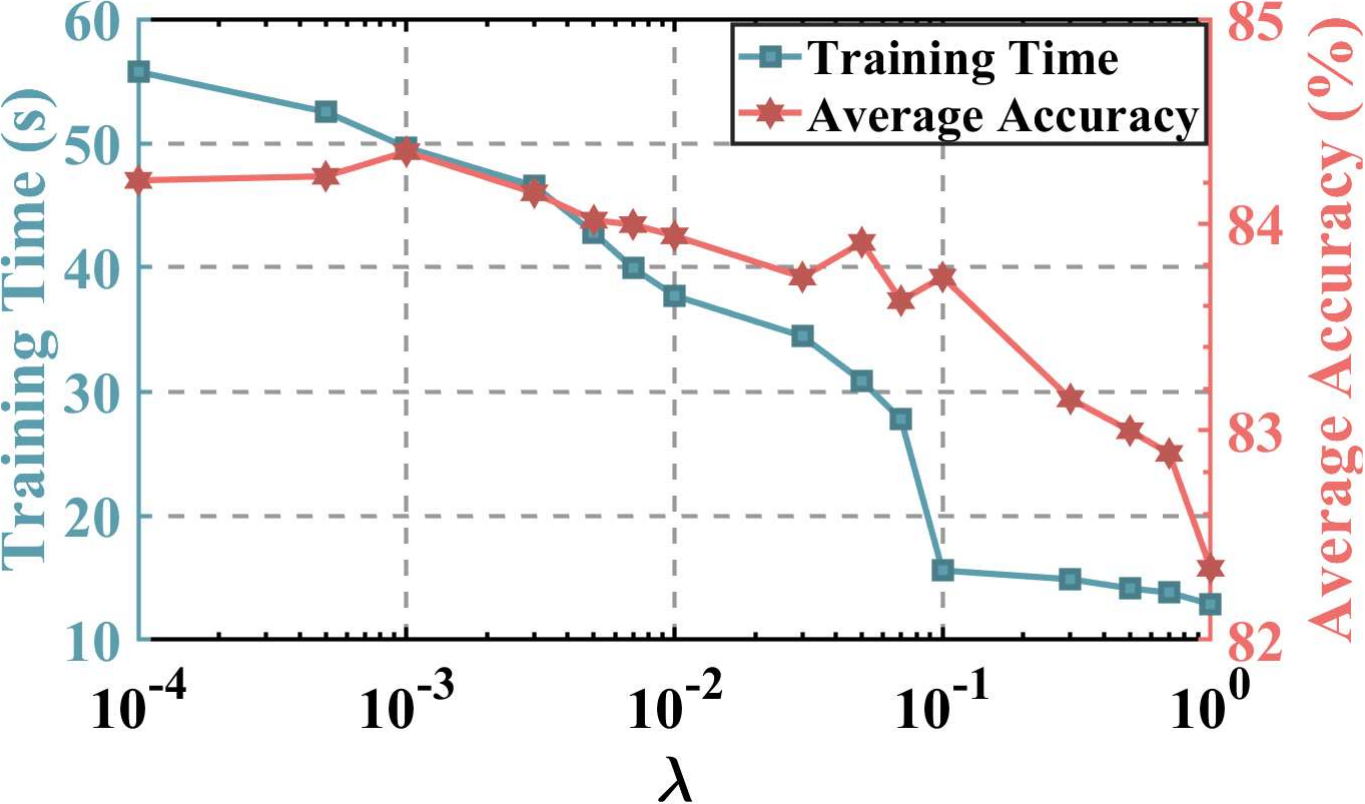}\label{fig:lambda_p}
\includegraphics[width=1.57in]{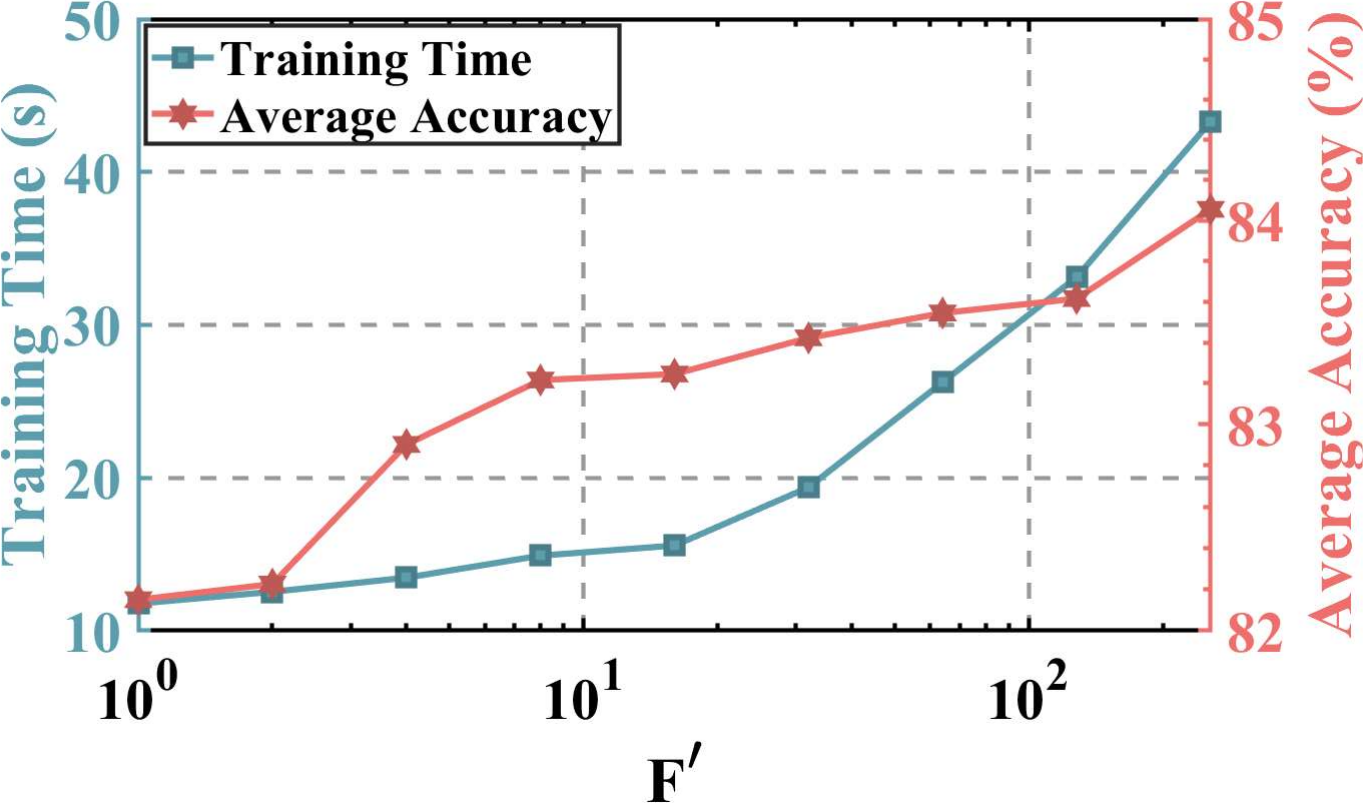}\label{fig:F_p}
}
\centering
 
\caption{\small The average training time and the average accuracy when setting different $\lambda$ and $F'$ on \textit{DBLP} and \textit{Patent} dataset.}\label{fig:lambda and F}

\end{figure}

\subsection{Information compression under flexible arrival patterns. }

In the experiments above, we added edges to the initial graph following an edge-inclined arrival pattern. However, using a more flexible arrival pattern would allow for a more thorough examination of the models' memory capabilities, particularly in assessing whether they experience significant performance degradation. Therefore, we employ more edge arrival patterns on each model and report the corresponding accuracy on \textit{Reddit} dataset, as depicted in Fig. \ref{fig:pattern}. In the edge-declined pattern, we reverse the order of the edge-inclined setting and remove the batch of edges sequentially from the final snapshot. Then, we randomly add and remove one batch of edges simultaneously to the initial graph, which is denoted as the edge-balanced pattern. In the edge-declined pattern, we observe that as the edges are missing gradually, all methods suffer from a performance decrease. However, {\sc Coden} is only slightly affected by the removal of edge batches, which particularly achieves a non-trivial improvement on the initial graph compared with the edge-inclined pattern at time step $16$ (e.g., $91.86\%$ vs $87.86\%$). Moreover, {\sc Coden} demonstrates a significant improvement in this pattern, achieving an average prediction accuracy of $93.48\%$, indicating its ability to effectively compress historical information for future predictions.

\begin{figure}[htb]
    \centering
\includegraphics[width=3.35in]{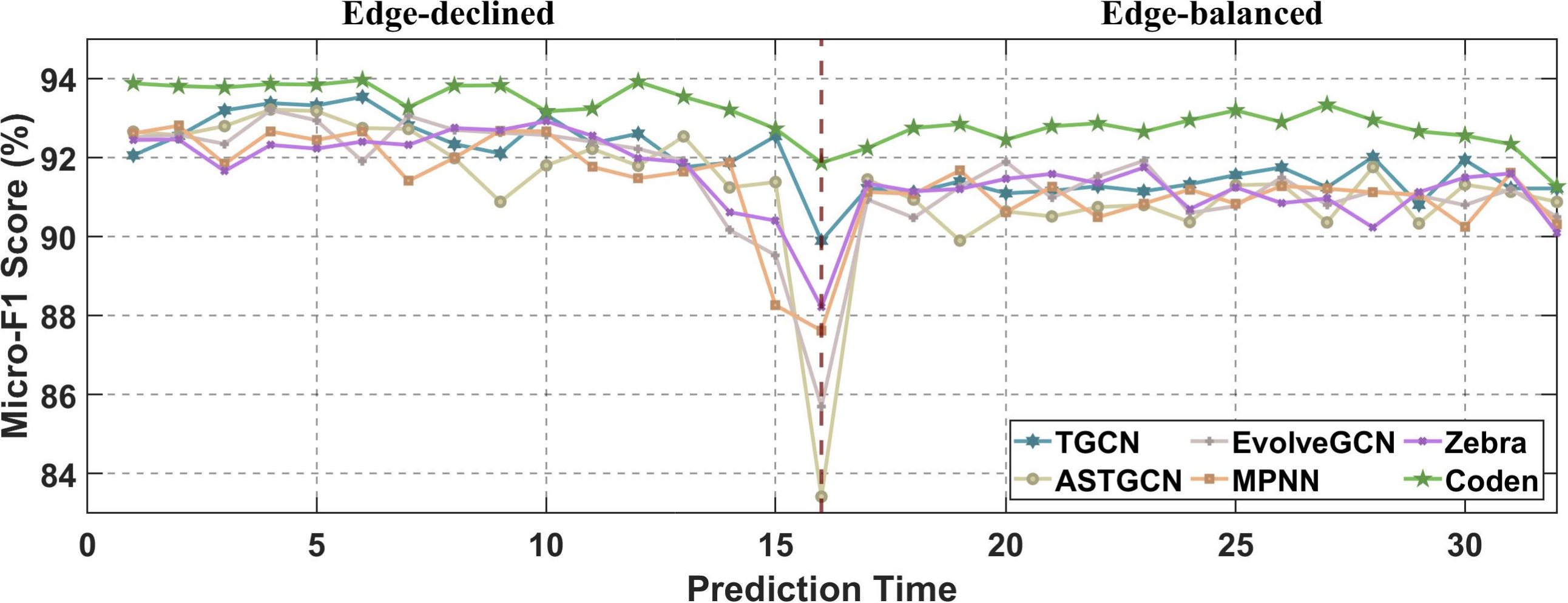}
    \caption{\small Micro-F1 scores under the Edge-declined and Edge-balanced patterns on \textit{Reddit} dataset.}

    \label{fig:pattern}
\end{figure}

\section{Conclusion}
In this work, we present {\sc Coden}, an innovative framework intending to improve the performance of TGNNs in a dynamic scenario requiring continuous predictions. {\sc Coden} performs a unique state-updating mechanism, where the node embeddings are updated incrementally and facilitate the compression
of historical information in the state. Theoretically, {\sc Coden} can provide an accuracy guarantee after the embedding update and approximately achieve information compression of the continuously evolving process.  Our extensive experimental results demonstrate {\sc Coden}'s superior performance in both efficacy and efficiency when compared with state-of-the-art methods.

\bibliographystyle{IEEEtran}
\bibliography{sample}

\begin{IEEEbiography}[ \raisebox{0.22\height}{\includegraphics[width=1in,height=1.25in,clip,keepaspectratio]{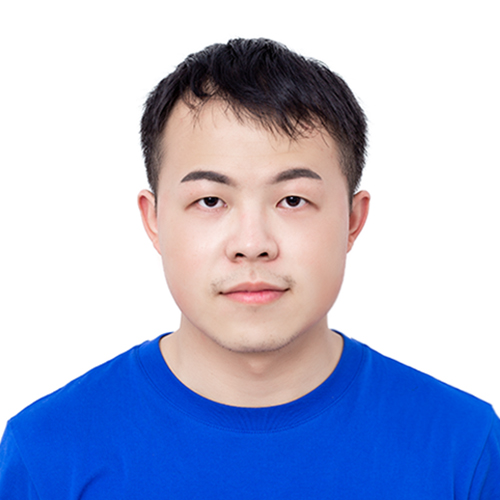}}]{Zulun Zhu} received the bachelor’s degree from Zhejiang University and the master’s degree from the National University of Defense Technology. He is currently a fourth-year PhD candidate with the College of Computing and Data Science, Nanyang Technological University. His research interests include databases for graphs, machine learning, graph neural networks, and spiking neural networks.
\end{IEEEbiography}

\begin{IEEEbiography}[ \raisebox{0.22\height}{\includegraphics[width=1in,height=1.25in,clip,keepaspectratio]{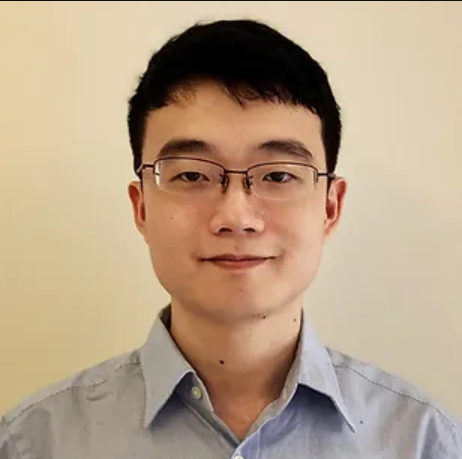}}]{Siqiang Luo}
is currently an assistant professor with the College of Computing and Data Science, Nanyang Technological University. He
was a postdoc at Harvard University from 2019 to 2020. He received his Ph.D. degree in computer science from the University of Hong Kong
in 2019. His research interest includes scalable data structures and systems, as well as graph analytics and mining.
\end{IEEEbiography}
\iftrue

\clearpage
\newpage

\appendix

\section{Appendix}

\subsection{Proof of Lemma \ref{lemma:single shift}}
\begin{proof}
    According to the definition of 1-Norm of the matrix, we have:
    {\footnotesize
    \begin{align*}
        &||{\bm H}^{(t+1)} - {\bm H}^{(t)}||_1 = \max_{f\in F}\sum_u|{\bm h}^{(t+1)}(u))-{\bm h}^{(t)}(u))|\\
        &\leq \max_{f\in F}\sum_u|{\bm z}^{(t+1)}(u)-{\bm z}^{(t)}(u))+2\epsilon|
    \end{align*}}
We can express ${\bm Z}^{(t)}$ as the result of global smoothness \cite{ma2021unified, unifews} which satisfies:
{\footnotesize
\begin{align} \label{equ:laplacian}
    {\bm Z}^{(t)} = \left({\bm I}+c {\bm L}^{(t)}\right)^{-1} {\bm X}^{(t)},
\end{align}}
where $c = \frac{1}{\alpha} - 1$ and ${\bm L}^{(t)} = {\bm I} - {\bm P}^{(t)}$ is the Laplacian matrix and denotes the topological status of graph $\mathcal{G}^{(t)}$. Hence we have:
{\footnotesize
   \begin{align*}
        &||{\bm H}^{(t+1)} - {\bm H}^{(t)}||_1 \\
        &\leq \max_{f\in F}\left\lVert\left( \left({\bm I}+c {\bm L}^{(t+1)}\right)^{-1}  - \left({\bm I}+c {\bm L}^{(t)}\right)^{-1}\right){\bm x}^{(t)}\right\rVert_1  +2n\epsilon\\ 
&\leq c\max_{f\in F}\left\lVert\left({\bm I}+c {\bm L}^{(t+1)}\right)^{-1}\left( {\bm L}^{(t+1)} - {\bm L}^{(t)}\right)\left({\bm I}+c {\bm L}^{(t)}\right)^{-1}{\bm x}^{(t)}\right\rVert_1 \\
        &+2n\epsilon\leq c \max_{f\in F}\left\lVert\left( {\bm P}^{(t+1)} - {\bm P}^{(t)}\right){\bm x}^{(t)}\right\rVert_1  +2n\epsilon \\
        &\leq \frac{1-\alpha}{\alpha}\left\lVert\Delta {\bm P}\cdot{\bm x}_{max}\right\rVert_1 +2n\epsilon,
    \end{align*}
}
where ${\bm x}_{max}$ is the row-wise maximum absolute value vector and $i$-th entry of ${\bm x}_{max}$ is defined as: $\{{\bm x}_{max}\}_i = \max_{1\leq j\leq n}|{\bm X}_{ij}|$. Proof finished.

\end{proof}

\subsection{Proof of Proposition \ref{pro:approx}}

\begin{proof}
  We assume there come $p$ edge updates between $t$ and $t+1$. Since we can not access the concrete timestamps of these $p$ edges, we can assume $k$-th each edge arrives at the system at a hidden timestamp $t+\tau_k$ ($1 \leq k \leq p$) without loss of generality. Specifically, we have $t+\tau_p = t+1$. According to \cite{graphssm}, the complete evolving process from $t$ to $t+1$ can be formulated as:
  {\footnotesize
  \begin{align*}
      {\bm M}^{(t+1)} &= \bar{\mathcal{A}} {\bm M}^{(t)} + \bar{\mathcal{B}} \sum_{k=1}^{p}\sum_{l=0}^{\infty} \alpha\left((1-\alpha)^l {\bm P}^{(t+\tau_k)}\right)^l {\bm X} {\Lambda}_{k}\\
      & = \bar{\mathcal{A}} {\bm M}^{(t)} + \bar{\mathcal{B}} \sum_{k=1}^{p}\alpha\left({\bm I}-(1-\alpha){\bm P}^{(t+\tau_k)}\right)^{-1} {\bm X} {\Lambda}_{k}.
  \end{align*}}
When setting $\alpha = \frac{1}{c+1}$, we have:
{\footnotesize
\begin{align*}
    {\bm M}^{(t+1)} &= \bar{\mathcal{A}} {\bm M}^{(t)} + \bar{\mathcal{B}} \sum_{k=1}^{p} \left({\bm I}+c {\bm L}^{(t+\tau_k)}\right)^{-1} {\bm X} {\Lambda}_{k},
\end{align*}}
where ${\bm L}^{(t)}$ is the Laplacian matrix at time $t$.
Given the exact PPR embedding $\left({\bm I}+c {\bm L}^{(t+1)}\right)^{-1} {\bm X}$ at time $t+1$, we have:
{\tiny
\begin{align*}
    &\left\lVert\left({\bm I}+c {\bm L}^{(t+1)}\right)^{-1} {\bm X} - \sum_{k=1}^{p} \left({\bm I}+c {\bm L}^{(t+\tau_k)}\right)^{-1} {\bm X} {\Lambda}_{k}\right\rVert\\
    &= \left\lVert\left({\bm I}+c {\bm L}^{(t+1)}\right)^{-1} {\bm X} \sum_{k=1}^{p} {\Lambda}_{k} - \sum_{k=1}^{p} \left({\bm I}+c {\bm L}^{(t+\tau_k)}\right)^{-1} {\bm X} {\Lambda}_{k}\right\rVert\\
    &=\left\lVert \sum_{k=1}^{p} \left( \left({\bm I}+c {\bm L}^{(t+1)}\right)^{-1}  - \left({\bm I}+c {\bm L}^{(t+\tau_k)}\right)^{-1} \right) {\bm X} {\Lambda}_{k}\right\rVert\\
    &\leq \left\lVert\sum_{k=1}^{p} \left({\bm I}+c {\bm L}^{(t+1)}\right)^{-1} \left(c {\bm L}^{(t+1)} -c {\bm L}^{(t+\tau_k)}\right)\left({\bm I}+c {\bm L}^{(t+\tau_k)}\right)^{-1} {\bm X} {\Lambda}_{k}\right\rVert\\
    &\leq c \left\lVert\sum_{k=1}^{p}\left({\bm I}+c {\bm L}^{(t+1)}\right)^{-1}\left( {\bm L}^{(t+1)} - {\bm L}^{(t+\tau_k)}\right)\left({\bm I}+c {\bm L}^{(t+\tau_k)}\right)^{-1}{\bm X}\right\rVert \left\lVert\sum_{k=1}^{p} {\Lambda}_{k}\right\rVert\\
    &\leq c\sum_{k=1}^{p} \frac{\left\lVert\left( {\bm L}^{(t+1)} - {\bm L}^{(t+\tau_k)}\right){\bm X}\right\rVert_2}{\left(1+c\lambda_1({\bm L}^{(t+1)})\right)\left(1+c\lambda_1({\bm L}^{(t+\tau_k)})\right)} \\
    &\leq c\sum_{k=1}^{p}\left\lVert\left( {\bm L}^{(t+1)} - {\bm L}^{(t+\tau_k)}\right){\bm X}\right\rVert \leq \lambda
\end{align*}}

\end{proof}

\subsection{Proof of Lemma \ref{lem:coden_gru_form}}

\begin{proof}
ZOH discretization maps continuous parameters $(\Delta,\mathcal A,\mathcal B)$ to discrete parameters
$(\bar{\mathcal A},\bar{\mathcal B})$ as
\begin{align*}
\bar{\mathcal A} &= \exp(\Delta\mathcal A), \\
\bar{\mathcal B} &= (\Delta\mathcal A)^{-1}\bigl(\exp(\Delta\mathcal A)-\bm I\bigr)\,\Delta\mathcal B.
\end{align*}
Apply the channel-wise constraints $\mathcal A=-\bm I$ and $\mathcal B=\bm I$, and denote the positive step size at time $t$ by $\Delta^{(t)}$.
Then
\begin{align*}
\bar{\mathcal A}^{(t)}
&= \exp\bigl(\Delta^{(t)}(-\bm I)\bigr)
= \exp(-\Delta^{(t)})\bm I .
\end{align*}
For $\bar{\mathcal B}^{(t)}$, note that $\Delta^{(t)}\mathcal A=-\Delta^{(t)}\bm I$, hence
$(\Delta^{(t)}\mathcal A)^{-1}=(-\Delta^{(t)})^{-1}\bm I$ and
$\exp(\Delta^{(t)}\mathcal A)-\bm I=\exp(-\Delta^{(t)})\bm I-\bm I$.
Therefore,
\begin{align*}
\bar{\mathcal B}^{(t)}
&= (-\Delta^{(t)})^{-1}\bigl(\exp(-\Delta^{(t)})\bm I-\bm I\bigr)\,\Delta^{(t)}\bm I \\
&= \bm I-\exp(-\Delta^{(t)})\bm I .
\end{align*}
Substituting into Equ.~\ref{equ:Coden_all} gives
\begin{align*}
\bm M^{(t+1)}
&= \exp(-\Delta^{(t+1)})\bm I\,\bm M^{(t)}
+\bigl(\bm I-\exp(-\Delta^{(t+1)})\bm I\bigr)\bm H^{(t+1)}.
\end{align*}
Let $\bm z^{(t+1)}:=\exp(-\Delta^{(t+1)})\bm I$. Since $\Delta^{(t+1)}>0$, we have $\bm z^{(t+1)}\in(0,1)$ and thus
\begin{align*}
\bm M^{(t+1)}
&= \bm z^{(t+1)}\bm M^{(t)}+\bigl(\bm I-\bm z^{(t+1)}\bigr)\bm H^{(t+1)}.
\end{align*}
Finally, if $\Delta^{(t)}=\mathrm{softplus}(s^{(t)})=\log(1+\exp(s^{(t)}))$, then
\begin{align*}
\bm z^{(t)}
&= \exp(-\Delta^{(t)})\bm I
= \exp\bigl(-\log(1+\exp(s^{(t)}))\bigr)\bm I \\
&= \frac{1}{1+\exp(s^{(t)})}\bm I
= \sigma(-s^{(t)})\bm I,
\end{align*}
which completes the proof.
\end{proof}

\subsection{Proof of Lemma \ref{lemma:emb complexity}}
\begin{proof}
    We first discuss the simple situation where only one edge $(u,v)$ is added into graph $\mathcal{G}^{(t)}$ and transfer it into $\mathcal{G}^{(t+1)}$. Then we can naturally derive the propagation complexity of a single dimension as $\|(1-\alpha)\left( {\bm P}^{(t+1)}- {\bm P}^{(t)}\right)$ $ {\bm h}^{(t)}\|_1/\alpha\epsilon$,
where $ {\bm h}^{(t)}$ is one dimension of $ {\bm H}^{(t)}$. As demonstrated in Alg. \ref{alg:increment}, given an added edge $(u,v)$  the residue of $\|(1-\alpha)\left( {\bm P}^{(t+1)}- {\bm P}^{(t)}\right)$ $ {\bm h}^{(t)}\|_1$ can be divided as: 
{\footnotesize
\begin{align*}
    &{\|(1-\alpha)\left( {\bm P}^{(t+1)}- {\bm P}^{(t)}\right) {\bm h}^{(t)}\|_1} 
    = \left|(1-\alpha){{\bm {h}}^{(t)}(u)}/{|\mathcal{N}_{out}^{(t+1)}(u)|}\right|\\& + \left|\sum_{w\in \mathcal{N}_{out}^{(t+1)}(u)}  (1-\alpha){{\bm {h}}^{(t)}(u)}\left({1}/{|\mathcal{N}_{out}^{(t+1)}(u)|} -{1}/{|\mathcal{N}_{out}^{(t)}(u)|} \right)\right| \\
    &{\leq}  (1-\alpha)| {\bm h}^{(t)} (u)| + \left| |\mathcal{N}_{out}^{(t+1)}(u)| \cdot  (1-\alpha){{\bm {h}}^{(t)}(u)}\cdot {1}/{|\mathcal{N}_{out}^{(t+1)}(u)|}  \right|\\
    &\overset{(1)}{\leq} 2 (1-\alpha)| {\bm \phi}^{(t)} (u)|+2 (1-\alpha) \epsilon,
\end{align*}
}
where ${\bm \phi}^{(t)} = \alpha\left(I-(1-\alpha){\bm P}^{(t)}\right)^{-1} {\bm x}^{(t)}$ and (1) holds since ${{\bm {h}}^{(t)}(u)}$ is the underestimation of $ {\bm \phi}^{(t)}(u)$. Under pattern (i), edges arrive in the system randomly so that each node has the same probability of being the start point of the changed edge. Therefore we can obtain the expected time for each update as 
 $E[\frac{2 (1-\alpha)| {\bm \phi}^{(t)} (u)|+2 (1-\alpha) \epsilon}{\alpha\epsilon}] = \frac{2(1-\alpha) ||{\bm x}^{(t)}||_1}{\epsilon n^{(t)}}+\frac{2(1-\alpha)}{\alpha}$.
Given there exist $p$ edge updates in a batch at time $t_k$, we add the superscript of ${\bm x}^{(t_k)}$ and formulate the complexity for propagation as $O(p\sum_{i=1}^F\frac{||{\bm x}_i^{(t_k)}||_1}{\epsilon n^{(t_k)}})$. Under pattern (ii), we have $\frac{2 (1-\alpha)| {\bm \phi}^{(t)} (u)|}{\alpha\epsilon}\leq \frac{2 (1-\alpha)\max_{u\in \mathcal{V}}|{\bm x}^{(t)}(u)|}{\epsilon }$, which indicates the complexity of this case is 
$O(p\sum_{i=1}^F\frac{\max_{u\in \mathcal{V}}|{\bm x}_i^{(t_k)}(u)|}{\epsilon})$
\end{proof}

\subsection{Proof of Lemma \ref{lemma: up state}}
\begin{proof}

Based on Alg. \ref{alg:Coden}, for an edge update $(u,v)$, the increased mass of $\sigma$ is at most $\frac{1-\alpha}{\alpha}\left\lVert\left( {\bm P}^{(t)} - {\bm P}^{(t-1)}\right)\cdot{\bm x}_{max}\right\rVert_1 +2n\epsilon$. Following the derivation of Lemma \ref{lemma:emb complexity}, we can naturally obtain the upper bound of this term as $2 (1-\alpha)| {\bm x}_{max} (u)|+2n\epsilon$. 
Therefore, the largest number of edges contained before $\sigma >\lambda$ is at least $\lambda/(2 (1-\alpha)| {\bm x}_{max} (u)|+2n\epsilon)$. Then given $p$ update events, the number of embeddings sampled $L$ can be bounded as $L \leq (2 (1-\alpha)| {\bm x}_{max} (u)|+2n\epsilon)p/\lambda$. Finally, the time complexity of updating the node state in SSM is $O(n^{(t)}LF^2)$ \cite{graphssm,mamba}. Therefore, we can naturally obtain the complexity of updating node states as $O(\frac{||{\bm x}_{max}||_1+\epsilon\left(n^{(t)}\right)^2}{\lambda}pF^2)$ under pattern (i) and $O(\frac{ \max_{u\in \mathcal{V}^{(t)}}|{\bm x}_{max}(u)| n^{(t)}+\epsilon\left(n^{(t)}\right)^2}{\lambda}pF^2)$ under pattern (ii).
\end{proof}

\subsection{Proof of Proposition \ref{lemma:energy}}

\begin{proof}



Based on the definition of the Dirichlet Energy and simplifying ${\bm I}-{{\bm A}^{(t)}}^{\top}{{\bm D}^{(t)}}^{-1}$ as $\Delta$, we have:
{\tiny
\begin{align*}
    \mathbb{DE}(\boldsymbol{M}^{(t)}_A) = \mathbb{DE}\left(\mathrm{softmax}\left(\frac{\left(\boldsymbol{W}_q\boldsymbol{H}^{(t)}\right)\cdot\left(\boldsymbol{W}_k\boldsymbol{H}^{(s)}\right)^\top}{\sqrt{F^{\prime}}}\right)_{0:t} \cdot \left(\boldsymbol{W}_v\boldsymbol{H}^{(s)}\right)_{0:t}\right).
\end{align*}
}
To simplify the demonstration, we denote 
{\footnotesize
\begin{align*}
{\bm S}_t = \mathrm{softmax}\left({\left(\boldsymbol{W}_q\boldsymbol{H}^{(t)}\right)
\cdot\left(\boldsymbol{W}_k\boldsymbol{H}^{(s)}\right)^\top}\right)_{0:t}. 
\end{align*}
}
Since ${\bm S}_t$ is a Row-Stochastic matrix where the sum of each row equals to 1, we can obtain the maximum eigenvalue $\sigma_{max} ({\bm S}_t)$ of ${\bm S}_t$ as  $\sigma_{max} ({\bm S}_t)\leq 1$.
Then we have:
{\footnotesize
\begin{align*}
    \mathbb{DE}(\boldsymbol{M}^{(t)}_A) &= tr\left({\bm S}_t \left(\boldsymbol{H}^{(s)}\right)_{0:t} \Delta {\bm S}_t^{\top} \left( \boldsymbol{H}^{(s)}\right)_{0:t}^{\top} \right) /F\\
    & = tr\left( \left( \boldsymbol{H}^{(s)}\right)_{0:t} \Delta  \left( \boldsymbol{H}^{(s)}\right)_{0:t}^{\top} {\bm S}_t {\bm S}_t^{\top}\right) /F\\
    &\leq tr\left( \left( \boldsymbol{H}^{(s)}\right)_{0:t} \Delta  \left( \boldsymbol{H}^{(s)}\right)_{0:t}^{\top}\right) \sigma_{max} ({\bm S}_t {\bm S}_t^{\top}) /F\\
    &\leq tr\left( \left( \boldsymbol{H}^{(s)}\right)_{0:t} \Delta  \left( \boldsymbol{H}^{(s)}\right)_{0:t}^{\top}\right)/F \\
    &= \sum_{s=0}^t tr\left( \left( \boldsymbol{H}^{(s)}\right) \left( \boldsymbol{H}^{(s)}\right)^{\top}\Delta  \right)/F
\end{align*}}

Following the similar derivation, we can obtain the Dirichlet Energy $\mathbb{DE}(\boldsymbol{M}^{(t)}_C)$ as:
{\footnotesize
\begin{align*}
    \mathbb{DE}(\boldsymbol{M}^{(t)}_C) &\ge \sum_{s=0}^t tr\left( \prod_{i=s+1}^t \bar{\mathcal{A}}^{(i)} \bar{\mathcal{B}}^{(s)} \left( \boldsymbol{H}^{(s)}\right) \left( \boldsymbol{H}^{(s)}\right)^{\top}\Delta  \right) \\
    & \ge \sum_{s=0}^t \sigma_{min} (\prod_{i=s+1}^t \bar{\mathcal{A}}^{(i)} \bar{\mathcal{B}}^{(s)}) tr\left( \left( \boldsymbol{H}^{(s)}\right) \left( \boldsymbol{H}^{(s)}\right)^{\top}\Delta  \right)\\
    &\ge  \mathbb{DE}(\boldsymbol{M}^{(t)}_A),
\end{align*}}
which holds when $F \cdot \sigma_{min}^2 (\prod_{i=s+1}^t \bar{\mathcal{A}}^{(i)} \bar{\mathcal{B}}^{(s)})\ge 1$ for $0 \leq s \leq t$.

In the following, we provide an preliminary lemma to indicate how we make  $F \cdot \sigma_{min}^2 (\prod_{i=s+1}^t \bar{\mathcal{A}}^{(i)} \bar{\mathcal{B}}^{(s)})\ge 1$ hold for $0 \leq s \leq t$.

\begin{lemma}\label{lemma:spectral_floor}
For each time step $1 \le i \le T$,  let $\bar{\mathcal A}^{(i)}$ be a
learnable square matrix, and for each $0 \le s \le T-1$,  let
$\bar{\mathcal B}^{(s)}$ be another learnable square matrix.
Define
$P_{s,t} =
\prod_{i=s+1}^{t}\bar{\mathcal A}^{(i)}\, \bar{\mathcal B}^{(s)}, 0 \le s \le t \le T $.
Assume every factor satisfies $\sigma_{\min}(W)\ge\gamma>0$.
If $\gamma \ge F^{-1/(2T)}$, then for all $0 \le s \le t \le T$, 
$F\cdot\sigma_{\min}^{2}\bigl(P_{s,t}\bigr)\ge 1$.
\end{lemma}

\begin{proof}
For any two real matrices $A$ and $B$ of compatible shapes,
\begin{equation}
\sigma_{\min}(AB)\ge\sigma_{\min}(A)\sigma_{\min}(B),
\label{eq:submult}
\end{equation}
because $\|ABx\|_2=\|A(Bx)\|_2\ge\sigma_{\min}(A)\|Bx\|_2
      \ge\sigma_{\min}(A)\sigma_{\min}(B)\|x\|_2$ for every unit vector $x$.

Apply~\eqref{eq:submult} recursively to
$P_{s,t}=\prod_{i=s+1}^{t}\bar{\mathcal A}^{(i)}
  \bar{\mathcal B}^{(s)}, 0\le s\le t\le T,$ to obtain
\begin{equation}
\sigma_{\min}(P_{s,t})\ge
\prod_{i=s+1}^{t}\sigma_{\min}\bigl(\bar{\mathcal A}^{(i)}\bigr)
  \sigma_{\min}\bigl(\bar{\mathcal B}^{(s)}\bigr).
\end{equation}

Because every matrix $W \in \mathcal W=\{\bar{\mathcal A}^{(i)},\bar{\mathcal B}^{(s)}\mid 1\le i\le T,\,0\le s\le T-1\}$ satisfies that $\sigma_{\min}(W)\ge\gamma$,
$\sigma_{\min}(P_{s,t}) \ge\gamma^{\,t-s+1}.$
Since $t-s+1 \le T$, we need to guarantee that
$\sigma_{\min}(P_{s,t}) \ge \gamma^{T}.$
Therefore
$F \cdot\sigma_{\min}^{2}(P_{s,t})\ge F \cdot\gamma^{2T} \ge 1$ when $\gamma \ge F^{-1/(2T)}$.
\end{proof}

Because we set $\gamma = c\,F^{-1/(2T)} (c>1)$ and $\sigma_{\min}(W) $ will be close to $\gamma$, we can have $F \cdot\sigma_{\min}^{2}(P_{s,t}) >1$ in our practical experiments. 

\end{proof}

 {
We collect all learnable transition matrices in  
$\mathcal W
  =\bigl\{
      \bar{\mathcal A}^{(i)},\,
      \bar{\mathcal B}^{(s)}
      \mid
      1\le i\le T,\; 0\le s\le T-1
    \bigr\},$
where $T$ is the overall length of the edge sequence.  
Following the soft-orthogonality regularizer of \cite{huang2018orthogonal}, we
penalize every $W\in\mathcal W$ with
$\mathcal R_{\text{orth}}(W)=\bigl\|W^{\top}W-\gamma^{2}I\bigr\|_{F}^{2}, \gamma = c\,F^{-1/(2T)} (c>1)$.
This term pulls all singular values of parameters toward $\gamma$ therefore
guarantees in practice
$F\cdot\sigma_{\min}^{2}\!\Bigl(\textstyle\prod_{i=s+1}^{t}
     \bar{\mathcal A}^{(i)}\bar{\mathcal B}^{(s)}\Bigr)\ge 1, 0\le s\le t$ (see Lemma \ref{lemma:spectral_floor} for the detailed proof).}

\subsection{Related Work}

\subsubsection{Single-snapshot Methods}

The single-snapshot GNN methods directly utilize the node embedding as the node states, which will be rapidly updated based on the new interactions. Specifically, single-snapshot GNN methods aim to transmit the updates into node embeddings ${\bm H}^{(t+1)}$ from ${\bm H}^{(t)}$ with the minimal time cost following: 
{\footnotesize
\begin{align}
    {\bm H}^{(t+1)} = \text{MSG} \left({\mathcal{G}}^{(t)}, {\bm H}^{(t)},  e_{t+1}\right).
\end{align}}
To facilitate the incremental updating based on these pre-computed node embeddings, Instant \cite{zheng2022instant}, DynAnom \cite{guo2022subset} and IDOL \cite{IDOL} explore the invariant rules of graph propagation and conduct an invariant-based
algorithm of Personalized PageRank (PPR) to refresh the node embedding locally. The dominant update complexity of these methods can be bounded as $O(p\sum_{i=1}^F\frac{||{\bm x}_i^{(t)}||_1}{\epsilon n^{(t)}})$ given $p$ updates and $\epsilon$ approximation error \cite{zheng2022instant}.
However, only focusing on the current snapshot will miss significant interactions of past time steps, leading to sub-optimal prediction results. 


\subsubsection{RNN-based Methods} 
RNN-based methods are generally based on the classical RNN \cite{cho2014learning}, GRU \cite{chung2014empirical}, LSTM \cite{hochreiter1997long} algorithm, etc., which simply use both the current input and the previous hidden state to iteratively capture temporal dependencies. For example, the GRU algorithm updates the node states at time $t+1$ as:
{\small
\begin{align*}
&  {\bm Z}^{(t+1)}=\operatorname{sigmoid}\left({\bm W}_Z  {\bm H}^{(t+1)}+{\bm U}_Z  {\bm M}^{(t)}+{\bm B}_Z\right) \\
& {\bm R}_t=\operatorname{sigmoid}\left({\bm W}_R {\bm H}^{(t+1)}+{\bm U}_R {\bm M}^{(t)}+{\bm B}_R\right) \\
& \widetilde{{\bm H}}^{(t+1)}=\tanh \left({\bm W}_H {\bm H}^{(t+1)}+{\bm U}_H\left({\bm R}_t \circ {\bm M}^{(t)}\right)+{\bm B}_H\right) \\
& {\bm M}^{(t+1)}=\left(1- {\bm Z}^{(t+1)}\right) \circ {\bm M}^{(t)}+ {\bm Z}^{(t+1)} \circ \widetilde{{\bm H}}^{(t+1)},
\end{align*}}
where ${\bm W}, {\bm U}, {\bm B} $ denote the trainable parameters of the linear layer and ${\bm H}^{(t+1)}$ can be updated using the $\text{MSG}(\cdot)$. Specifically,
TGCN \cite{tgcn} and EvolveGCN \cite{evolvegcn} and incorporate the Graph Convolutional Network (GCN)\cite{DBLP:conf/iclr/KipfW17} as the $\text{MSG}(\cdot)$ function to regenerate the node embeddings while coupling with an RNN-based module to learn temporal node representations. Similarly, MPNN \cite{MPNN} transforms the node embeddings at different time steps into an RNN module and then captures the long-range dependency in the final hidden state. This category of method generally requires $O(Km^{(t)}F)$ and $O(pn^{(t)}F^2)$ to update the node embeddings and states \footnote{For a clear presentation, we assume the dimension of node state $F' = F$}, respectively. Due to their iterative structure, RNN-based methods can efficiently update node states. However, the simplistic recurrence mechanism often leads to difficulties in retaining historical information, especially as the graph size and temporal scope expand \cite{graphssm, gers2000learning}.

\subsubsection{Attention-based Methods} To address the forgetting problem of RNNs, attention-based methods rely on the attention mechanism and abstain from using recurrence form, which encodes the position of sequences and enables the efficient information flow from past to current representations. We take the representative work APAN \cite{wang2021apan} as an example to demonstrate the core mechanism of this category.
Considering matrices ${\bm Q}\in \mathbb{R}^{n\times F}$ denoted as "query", ${\bm K}\in \mathbb{R}^{n\times F}$ denoted as "keys", and ${\bm V}^{n\times F}$ denoted as "values", the classical attention algorithms perform the following computation to obtain the optimized embeddings:
{\footnotesize
\begin{align*}
    &{\bm M}^{(t+1)} = \mathrm{Attn}({\bm Q},{\bm K},{\bm V})=\mathrm{softmax}\left(\frac{{\bm Q \bm K}^{\top}}{\sqrt{F}}\right){\bm V},\\
    &{\bm Q} = {\bm H}^{(t+1)} {\bm W}_q,  {\bm K} = {\bm M}^{(t)} {\bm W}_k,
    {\bm V} = {\bm M}^{(t)} {\bm W}_v,
\end{align*}}
where ${\bm W}_q, {\bm W}_k, {\bm W}_ \in \mathbb{R}^{ F\times F'}$ are the network parameters. The dot-product term $\left(\frac{{\bm Q \bm K}^{\top}}{\sqrt{F}}\right)$ takes the role of weighting the interactions between entity "query-key" pairs. A higher value within this term increases the contribution of ${\bm V}$ to the embedding space. Thus, attention-based methods create the expressive attention score to capture the relationship between the current embedding and the state of the last time step. Following this intuition, DySat \cite{sankar2020dysat} employs the generalized GAT module \cite{velivckovic2018graph} to integrate the embeddings of a single node from different time steps to generate its refreshed one. ASTGCN \cite{ASTGCN} and DNNTSP \cite{yu2020predicting} further incorporate the attention mechanism to capture the spatial and temporal dependency for enhanced representation quality. For each time step, these methods generally require $O(T\left(n^{(t)}\right)^2F)$ time complexity to calculate the final output given $T$ time step.
While attention mechanisms can retain the most relevant parts of the sequences to avoid the forgetting issue, they can become computationally expensive with frequent updates \cite{thomasgraph}.

\subsubsection{Other TGNN methods.} (i) \textit{SNN-based methods.} Another typical mechanism of this category is based on the biological Spiking Neural Networks (SNNs), which simulate the brain behaviors and maintain the membrane potential given the data sequences. SpikeNet \cite{spikenet} retrieves the node embeddings from multiple time steps and finally generates the prediction results by the spike firing 
process. Dy-SIGN \cite{yin2024dynamic} incorporates SNNs mechanism into dynamic graphs to mitigate the information loss and memory consumption problem. Nevertheless, the demand for multiple simulation steps to generate reliable embeddings can significantly degrade the efficiency of these TGNN methods. (ii) \textit{SSM-based methods.} There are also some works which employs the SSM mechanism on temporal graphs.
For example, STG-Mamba \cite{li2024stg} formulates the feature stream of each node as the long-term context, which improves the embedding quality for feature-varying graphs. Graph-SSM \cite{graphssm} addresses the unobserved graph mutations between consecutive snapshots, and achieves an effective discretization with long-term information. However, these methods only focus on the discrete-time dynamic graph and fail
to model the continuous topology changes as the graph evolves. Furthermore, directly adapting these methods will incur significant computational overhead as the graph evolves, creating a gap between current algorithms and continuous prediction in practical dynamic scenarios.

\subsubsection{Forward Push}\label{app:forward}
 The prevailing paradigm for existing PPR-based propagation is derived from the concept of \textit{Forward Push}, which calculates the approximated solution ${\bm z}_i =\sum_{l=0}^{\infty} \alpha(1-\alpha)^l \left({{\bm A}^{(t)}}^{\top}{{\bm D}^{(t)}}^{-1}\right)^l  {\bm x}_i^{(t)}$ ($1\leq i \leq F$) under a given error bound $\epsilon$, where $\alpha$ is the decay factor of random walk. Specifically, \textit{Forward Push} (depicted in Alg. \ref{alg:forward}) assigns the attribute vector ${\bm x}_i^{(t)}$ to the \textit{residue vector} ${\bm r}^{(t)}$ (e.g., ${\bm r}^{(t)} = {\bm x}_i^{(t)}$), which represents the unpropagated mass of attribute vector ${\bm x}_i^{(t)}$. \footnote{In the following sections, we will omit the subscript $i$ for simplicity. } We iteratively conduct the following two steps: (1) For each node $s \in \mathcal{V}^{(t)}$ such that ${\bm r}^{(t)}(s)> \epsilon$, $(1-\alpha)$ fraction of ${\bm r}^{(t)}(s)$ will be propagated into the out-neighbors $t\in \mathcal{N}_{out}(s)$ averagely. (2) $(1-\alpha)$ fraction of ${\bm r}^{(t)}(s)$ will be transferred into the reserve vector ${\bm h}^{(t)}$ and then ${\bm r}^{(t)}(s)$ is set as 0. This iteration will be terminated until ${\bm r}^{(t)}(s)\leq \epsilon$ for all nodes $s \in \mathcal{V}^{(t)}$ and we can deploy ${\bm h}^{(t)}$ as the approximated node embedding which satisfies $|{\bm h}^{(t)} (s)- {\bm z}^{(t)}(s)| \leq \epsilon$ for each node $s \in \mathcal{V}^{(t)}$. 

\IncMargin{1em}
\begin{algorithm}[tb]
\small
   \SetKwInOut{Input}{Input}\SetKwInOut{Output}{Output}
   \Input{Graph $\mathcal{G}^{(t)} = (\mathcal{V}^{(t)},\mathcal{E}^{(t)})$, { reserve vector ${\bm h}^{(t)}$,}  residue vector ${\bm r}^{(t)}$.}
   
   \Output{Reserve vector ${\bm {h}}^{(t)}$}
${\bm h}^{(t)} = {\bm h}^{(t-1)}$\;   
\While{exists $s \in \mathcal{V}$ such that ${\bm r}^{(t)}(s)> \epsilon $}{
\ForEach{$v\in \mathcal{N}_{out}^{(t)}(s)$}{
${\bm r}^{(t)}(v) += (1-\alpha)\cdot \frac{{\bm r}^{(t)}(s)}{|\mathcal{N}_{out}^{(t)}(s)|}$;
}

${{\bm h}^{(t)}}(s)+=\alpha \cdot {\bm r}^{(t)}(s)$;
${\bm r}^{(t)}(s)=0$\;
}

\caption{ {Forward Push}}\label{alg:forward}
 
\end{algorithm}\DecMargin{1em}

\subsection{Introduction of baseline methods}\label{app:listed baseline}
\subsubsection{Single-snapshot methods}
\textbf{Instant \cite{zheng2022instant}, DynAnom \cite{guo2022subset}, IDOL \cite{IDOL}. } These three methods inherit the updating skeleton of PPR \cite{zhang2016approximate} to incrementally update the node embeddings.  Following the analysis of \cite{zheng2022instant}, the dominant complexity of updating embeddings can be formulated as $O\left(\frac{||{\bm x}^{(t)}||_1}{\epsilon n}\right)$ for a feature vector ${\bm x}^{(t)}$ assuming $\alpha$ as the constant. Hence we summarize the complexity of this catogory as $O\left(p\sum_{i=1}^F{||{\bm x}_i^{(t)}||_1}\right)$ for $F$ dimensions given $\epsilon = \frac{1}{n}$.

\subsubsection{RNN-based methods}
\textbf{TGN \cite{tgn}.} TGN proposes a general framework for learning the representation in CTDG, which models the past interactions between nodes using a compressed node state vector. Based on existing techniques, TGN provides flexible modules to compute the embeddings and update node states. For example, TGN utilizes the attention mechanism and GCN \cite{DBLP:conf/iclr/KipfW17} to compute the node embeddings and LSTM or GRU module to update the node state. 

\textbf{TGCN \cite{tgcn}, EvolveGCN \cite{evolvegcn}, MPNN \cite{MPNN}, ROLAND \cite{you2022roland}.}  These four methods employ a unified pipeline to update the node embeddings and states. Specifically, these methods conduct the graph propagation based on GCN \cite{DBLP:conf/iclr/KipfW17} and update the state with GRU or LSTM units. Since each update will require to compute the embeddings from scratch, we summarize their complexity as $O(Km^{(t)}F)$ and $O(pn^{(t)}F^2)$ given $p$ updates.

\subsubsection{Attention-based methods} 

\textbf{DySat \cite{sankar2020dysat}, ASTGCN \cite{ASTGCN}.} DySat and ASTGCN utilize structural and temporal attention mechanism for dynamic graph representation learning. Both of them employ graph propagation by GCN.
As a result, the graph propagation and state update require $O(n^2F)$ time complexity for each update.

\textbf{TGAT \cite{xu2020inductive}.} TGAT adopts the GraphSage \cite{graphsage} for the embedding computation. Different from DySat, TGAT aims to formulate the interactions between the time encoding and node features. However, due to the polynomial property of the attention mechanism, TGAT still needs to consume $O(n^2F)$ time for state update.

\textbf{DNNTSP \cite{yu2020predicting}, SEIGN\cite{qin2023seign}, DyGFormer \cite{yu2023towards}.} These three methods follow a similar intuition to our alternative model {\sc Coden}-A, which intends to calculate the interactions between different time steps for each node. Given the time step $t$, this framework needs to consume $O(tn^2 F)$ time for the state update.

\subsection{Accuracy-guaranteed Embedding Update} \label{sec: incremental embed}

\textbf{Compensated Propagation.}
In this section, we discuss how to efficiently update the node embedding ${\bm H}^{(t)}$ and let $||{\bm H}^{(t)} - {\bm Z}^{(t)}||_1 \leq n^{(t)}\epsilon$ holds for each time step $t$ as demonstrated in Lemma \ref{lemma:single shift}. We draw on insights from recent works \cite{yoon2018fast, zheng2022instant, IDOL} to compensate for embedding distance across different time steps.

 Without loss of generality,
we start with an example that ${\bm x}^{(t+1)} = {\bm x}^{(t)} \in {\bm X}^{(t)}$ and a new edge $e_{t+1} = (u,v)$ is added at time $t+1$, creating an unbounded value difference between the old embedding ${\bm h}^{(t)}$ and the approximation target ${\bm z}^{(t+1)}.$ 
Note that ${\bm h}^{(t)}$ is the approximation of ${\bm z}^{(t)}$, and we have:
{\footnotesize
\begin{align}\label{equ:diff}
    &{\bm z}^{(t+1)}-{\bm z}^{(t)} \nonumber \\
    &= \alpha \left(\left(I-(1-\alpha){\bm P}^{(t+1)}\right)^{-1} - \left(I-(1-\alpha){\bm P}^{(t)}\right)^{-1} \right) {\bm x}^{(t+1)}  \nonumber \\
    & \overset{(1)}{=} \alpha (1-\alpha) \left(I-(1-\alpha){\bm P}^{(t+1)}\right)^{-1} \Delta{\bm P} \left(I-(1-\alpha){\bm P}^{(t)}\right)^{-1} {\bm x}^{(t+1)}  \nonumber \\
    & = \underbrace{\alpha \left(I-(1-\alpha){\bm P}^{(t+1)}\right)^{-1}}_{\textit{propagation process}} \cdot \underbrace{ (1-\alpha) \Delta{\bm P} {\bm z}^{(t)}}_{\textit{compensated vector}} 
\end{align}
}
where (1) holds since ${\bm U}^{-1}-{\bm V}^{-1}={\bm U}^{-1}({\bm V}-{\bm U}){\bm V}^{-1}$ and $\Delta{\bm P} = \left( {\bm P}^{(t+1)}- {\bm P}^{(t)}\right)$. Based on Equ. \ref{equ:diff}, we have a key observation: \textit{the difference between node embedding ${\bm z}^{(t)}$ and ${\bm z}^{(t+1)}$ can be compensated by propagating the new feature vector $(1-\alpha)\left( {\bm P}^{(t+1)}- {\bm P}^{(t)}\right) {\bm z}^{(t)} $ along the updated graph $\mathcal{G}^{(t+1)}$}. 
Since ${\bm h}^{(t)}$ is the approximation of ${\bm z}^{(t)}$ with the allowed error, we hence propose to implement the {embedding update} based on the \textit{compensated vector} $(1-\alpha) \Delta{\bm P} {\bm h}^{(t)}$.

\IncMargin{1em}
\begin{algorithm}[tb]
\small
   \SetKwInOut{Input}{Input}\SetKwInOut{Output}{Output}
   \Input{Graph $\mathcal{G}^{(t)} = (\mathcal{V}^{(t)},\mathcal{E}^{(t)})$, update events $\Gamma = \{(u_1,v_1), (u_2,v_2), ...., (u_p,v_p)\}$, reserve vector ${\bm h}^{(t)}$,  old adjacent matrix ${\bm P}^{(t)}$, updated adjacent matrix ${\bm P}^{(t+p)}$.}
   
   \Output{Reserve vector ${\bm {h}}^{(t+p)}$ }

${\bm r}^{(t+p)} = {\bm 0}$\;
$\mathcal{V}_{changed}\leftarrow\{u|\mathrm{~the~nodes~whose~out~degree}\mathrm{~has~changed}\}$\;

\ForEach{$u\in \mathcal{V}_{changed}$ \label{line:rev start}}{
    $\mathcal{N}_{add}(u)\leftarrow\{v|\mathrm{~the~added~neighbors~of~}u\}$\;
    $\mathcal{N}_{del}(u)\leftarrow\{v|\mathrm{~the~deleted~neighbors~of~}u\}$\;
    \ForEach{$v\in \mathcal{N}_{add}(u)$ \label{line:2_1}}{
    ${\bm r}^{(t+p)}(v) + = (1-\alpha){{\bm {h}}^{(t)}(u)}/{|\mathcal{N}_{out}^{(t+p)}(u)|}$\;
    }
    \ForEach{$v\in \mathcal{N}_{del}(u)$}{
        ${\bm r}^{(t+p)}(v) - = (1-\alpha){{\bm {h}}^{(t)}(u)}/{|\mathcal{N}_{out}^{(t)}(u)|}$\;
    }
    \ForEach{$w\in \mathcal{N}_{out}^{(t+p)}(u)\setminus\left(\mathcal{N}_{add}(u)\cup\mathcal{N}_{del}(u)\right)$}{
     \hspace{-2mm}${\bm r}^{(t+p)}(w) + = (1-\alpha){{\bm {h}}^{(t)}(u)}\left(\frac{1}{|\mathcal{N}_{out}^{(t+p)}(u)|} -\frac{1}{|\mathcal{N}_{out}^{(t)}(u)|} \right)$\label{line:2_2}\;
    }}   
$\textit{Forward Push }(\mathcal{G}^{(t+p)}, {\bm h}^{(t)}, {\bm r}^{(t+p)})$\label{line:2_3}\; 
\caption{ {Embedding Update}}\label{alg:increment}
\end{algorithm}\DecMargin{1em}

\textbf{Scalable Batch Update.} By adhering to the principle above, we extend the incremental update of node embedding into the batch setting, which is depicted in Alg. \ref{alg:increment}. Consider the graph $\mathcal{G}^{(t)}$ at time $t$ and the batch update events in $\Gamma = \{(u_1,v_1), (u_2,v_2), ...., (u_p,v_p)\}$ containing $p$ edge updates, where $(u_i,v_i)$ ($1\leq i \leq p$) will be viewed as deletion if it exists in $\mathcal{G}^{(t+i-1)}$ otherwise as an addition.
 First, we obtain the set $\mathcal{V}_{changed}$ denoting the nodes whose out-degree has changed. Given a node $u\in \mathcal{V}_{changed}$ and for each node $v\in \mathcal{N}_{out}^{(t)}(u) \cup \mathcal{N}_{out}^{(t+p)}(u)$, the corresponding entry of $ \Delta {\bm P} = {\bm P}^{(t+p)}- {\bm P}^{(t)}$ can be expressed as:
{\footnotesize
\begin{align*}
\Delta {\bm P}[v,u] = \left\{\begin{array}{cl}                       \frac{1}{|\mathcal{N}_{out}^{(t+p)}(u)|},        & v \text{ is the added neighbor of } u,         \\
                       -\frac{1}{|\mathcal{N}_{out}^{(t)}(u)|}, & v \text{ is the deteted neighbor of } u, \\
  \frac{1}{|\mathcal{N}_{out}^{(t+p)}(u)|}-\frac{1}{|\mathcal{N}_{out}^{(t)}(u)|}, &\text{otherwise}.         
                   \end{array}\right.
\end{align*}
}
Then, we multiply $\Delta {\bm P}[v,u]$ with $(1-\alpha){{\bm {h}}^{(t)}(u)}$ and finish the computation of the compensated vector $(1-\alpha)\left( {\bm P}^{(t+p)}- {\bm P}^{(t)}\right) {\bm h}^{(t)}$ (lines \ref{line:2_1}-\ref{line:2_2}). 
Moreover, we assign this new feature vector as the unpropagated residue vector ${\bm r}^{(t+p)}$. Note that ${\bm r}^{(t+p)}$ may exceed the permissible error for certain nodes, such as when ${\bm r}^{(t+p)}(u)> \epsilon $ ($u\in \mathcal{V}^{(t+p)}$). Consequently, we trigger the \textit{Forward Push} mechanism \cite{andersen2006local} again to propagate the residue vector ${\bm r}^{(t+p)}$ (line \ref{line:2_3}), which will subsequently influence the embedding values of other nodes and ensure the desired accuracy. Additionally, our algorithm can also be compatible with the attribute changes, which is detailed in Appendix \ref{app:attribute}.

\textbf{Theoretical Accuracy Guarantee.} Alg.  \ref{alg:increment} returns the updated embeddings by propagating the compensated vector along the updated graph, where the compensated vector is derived using ${\bm z}^{(t)}$, as specified in Equ. \ref{equ:diff}. Although we process this vector with the approximated vector ${\bm h}^{(t)}$, we can still establish an error bound for the output vector ${\bm h}^{(t+p)}$, which is the result of propagation from scratch at time $t+p$. We formally state this accuracy guarantee in the following lemma:
\begin{lemma}\label{lemma:update_error}
    Given  the normalized adjacent matrix ${\bm P}^{(t)}$ at time $t$ and the update event set $\Gamma = \{(u_1,v_1), (u_2,v_2), ...., (u_p,v_p)\}$, Alg \ref{alg:increment} can output the approximated embedding vector ${\bm {h}}^{(t+p)}$ which satisfies:
    {\footnotesize
    \begin{align*}
      ||{\bm {H}}^{(t+p)} - {\bm {Z}}^{(t+p)}||_1  \leq n^{(t)}\epsilon
    \end{align*}}
where ${\bm Z}^{(t+p)} =\sum_{l=0}^{\infty} \alpha(1-\alpha)^l \left({\bm P}^{(t+p)}\right)^l {\bm X}^{(t+p)}$.
\end{lemma}

\begin{proof}
 According to \cite{zheng2022instant}, the estimated vector ${\bm \hat{h}}^{(t-1)}$ at time $t-1$ and the exact PPR vector $\sum_{l=0}^{\infty} \alpha\left((1-\alpha)^l {\bm P}^{(t-1)}\right)^l \cdot {\bm x}_i$ have the following relationship:
 {\footnotesize
 \begin{align*}
     \sum_{l=0}^{\infty} \alpha\left((1-\alpha)^l {\bm P}^{(t-1)}\right)^l {\bm x}_i &= {\bm \hat{h}}^{(t-1)}+ \sum_{l=0}^{\infty} \alpha\left((1-\alpha)^l {\bm P}^{(t-1)}\right)^l {\bm r}_i\\
&= {\bm \hat{h}}^{(t-1)}+\alpha\left(I-(1-\alpha){\bm P}^{(t-1)}\right)^{-1}{\bm r}_i
 \end{align*}}
For the purpose of clear demonstration, we express $(1-\alpha) P^{(t)}$ as $P^{(t)}$ for the following proof.
Then we formulate the output ${\bm \hat{h}}^{(t)}$ of Alg. \ref{alg:increment} as:
{\footnotesize
\begin{align}\label{euq:app1}
  {\bm \hat{h}}^{(t)} &=  \sum_{j=0}^{\infty} (P^{(t)})^{j}(P^{(t)}-P^{(t-1)})\left.(\alpha\sum_{i=0}^{\infty}( P^{(t-1)})^{i}{\bm x}_i - \alpha\sum_{i=0}^{\infty}( P^{(t-1)})^{i}\right. \nonumber\\
&\left.{\bm r}_i^{(t-1)}
  \right.)+ {\bm \hat{h}}^{(t-1)}\nonumber\\
  &= \sum_{j=0}^{\infty} (P^{(t)})^{j}(P^{(t)}-P^{(t-1)})\alpha\sum_{i=0}^{\infty}( P^{(t-1)})^{i}{\bm x}_i - \nonumber\\
  &\sum_{j=0}^{\infty} (P^{(t)})^{j}(P^{(t)}-
  P^{(t-1)}) \alpha\sum_{i=0}^{\infty}( P^{(t-1)})^{i}{\bm r}_i^{(t-1)}+{\bm \hat{h}}^{(t-1)}
\end{align}
}

The first term of the last line in Equ. \ref{euq:app1} can be simplified as:
{\footnotesize
\begin{align*}
    &\sum_{j=0}^{\infty} (P^{(t)})^{j}(P^{(t)}-P^{(t-1)})\alpha\sum_{i=0}^{\infty}( P^{(t-1)})^{i}{\bm x}_i \\
    &= \sum_{j=1}^{\infty} (P^{(t)})^{j}\alpha\sum_{i=0}^{\infty}( P^{(t-1)})^{i}{\bm x}_i-\sum_{j=0}^{\infty} (P^{(t)})^{j}\alpha\sum_{i=j}^{\infty}( P^{(t-1)})^{i}{\bm x}_i\\
    &=\sum_{j=1}^{\infty} (P^{(t)})^{j}\alpha\left(\sum_{i=1}^{\infty}( P^{(t-1)})^{i}+{\bm I}\right){\bm x}_i-\\
    &\left(\sum_{j=1}^{\infty} (P^{(t)})^{j}+{\bm I}\right)\alpha\sum_{i=1}^{\infty}( P^{(t-1)})^{i}{\bm x}_i\\
    &=\alpha\sum_{j=1}^{\infty}( P^{(t)})^{i}{\bm x}_i - \alpha\sum_{i=1}^{\infty}( P^{(t-1)})^{i}{\bm x}_i \\
    &= \alpha\sum_{j=0}^{\infty}( P^{(t)})^{i}{\bm x}_i - \alpha\sum_{i=0}^{\infty}( P^{(t-1)})^{i}{\bm x}_i
\end{align*}
}
Similarly, the second term can be formed as:
{\footnotesize
\begin{align*}
    &\sum_{j=0}^{\infty} (P^{(t)})^{j}(P^{(t)}-P^{(t-1)}) \alpha\sum_{i=0}^{\infty}( P^{(t-1)})^{i}{\bm r}_i^{(t-1)} \\&= \alpha\sum_{j=0}^{\infty}( P^{(t)})^{i}{\bm r}_i^{(t-1)} - \alpha\sum_{i=0}^{\infty}( P^{(t-1)})^{i}{\bm r}_i^{(t-1)}
\end{align*}
}

Then Equ. \ref{euq:app1} is further expressed as:
{\footnotesize
\begin{align*}
  {\bm \hat{h}}^{(t)} &= \alpha\sum_{j=0}^{\infty}( P^{(t)})^{i}{\bm x}_i - \alpha\sum_{i=0}^{\infty}( P^{(t-1)})^{i}{\bm x}_i - \alpha\sum_{j=0}^{\infty}( P^{(t)})^{i}{\bm r}_i^{(t-1)} +\\
  &\alpha\sum_{i=0}^{\infty}( P^{(t-1)})^{i}{\bm r}_i^{(t-1)}+\alpha\sum_{i=0}^{\infty}( P^{(t-1)})^{i}{\bm x}_i - \alpha\sum_{i=0}^{\infty}( P^{(t-1)})^{i}{\bm r}_i^{(t-1)}\\
  &= \alpha\sum_{j=0}^{\infty}( P^{(t)})^{i}{\bm x}_i- \alpha\sum_{j=0}^{\infty}( P^{(t)})^{i}{\bm r}_i^{(t-1)}
\end{align*}
Since the ${\bm r}_i^{(t-1)}(u)<\epsilon$ for each $u \in \mathcal{V}^{(t-1)}$, the proof is finished.

}

\end{proof}

\subsection{Extension to attribute changes.}\label{app:attribute}
Note that we have made an assumption that ${\bm x}^{(t+p)} = {\bm x}^{(t)}$ such that we can focus on the topological change of the graph. Actually, this assumption can be easily relaxed since extending our approach to dynamic-attribute graphs is much more straightforward. Specifically, given $\Delta {\bm x} = {\bm x}^{(t+p)} - {\bm x}^{(t)}$ and ${\bm P}^{(t+p)} = {\bm P}^{(t)}$, we can naturally derive the difference between ${\bm z}^{(t)}$ and ${\bm z}^{(t+p)} $ as:
{\footnotesize
\begin{align}
  {\bm z}^{(t+p)} -{\bm z}^{(t)} =  \alpha \left(I-(1-\alpha){\bm P}^{(t+1)}\right)^{-1} \Delta {\bm x}.
\end{align}}
Therefore, we can view $\Delta {\bm x}$ as the compensated vector and easily obtain the updated embedding by directly invoking $\textit{Forward Push}$ $(\mathcal{G}^{(t+p)}, {\bm h}^{(t)}, \Delta {\bm x})$. \footnote{In the this paper, we mainly focus on the topology change and assume ${\bm X}^{(t)}$ remains unchanged in the following sections for a clear presentation.}

\fi
\end{document}